\documentclass{article}

\usepackage[preprint]{neurips_2024}

\usepackage[citecolor=black]{hyperref}
\usepackage[utf8]{inputenc} % allow utf-8 input
\usepackage[T1]{fontenc}    % use 8-bit T1 fonts
\usepackage{hyperref}       % hyperlinks
\usepackage{url}            % simple URL typesetting
\usepackage{booktabs}       % professional-quality tables
\usepackage{amsfonts}       % blackboard math symbols
\usepackage{nicefrac}       % compact symbols for 1/2, etc.
\usepackage{microtype}      % microtypography
\usepackage{xcolor}         % colors
\usepackage{url}
\usepackage{amsmath}
\usepackage{amssymb}
\usepackage{mathtools}
\usepackage{amsthm}
\usepackage{times}
\usepackage{latexsym}
\usepackage{url}
\usepackage{times}
\usepackage{latexsym}
\usepackage{subfigure}
\usepackage{graphicx}
\usepackage{amsmath}
\usepackage{mathrsfs}
\usepackage{amsfonts} 
\usepackage{bbm}
\usepackage{algorithm2e}
\usepackage{color}
\usepackage{multirow}
\usepackage{booktabs}
\usepackage{arydshln}
\usepackage{enumitem}
\usepackage{amssymb}
\usepackage{tabularx}
\usepackage{makecell}
\usepackage{fancyhdr}
\usepackage{wrapfig}
\usepackage{natbib}
\setcitestyle{numbers,square}
\hypersetup{
	colorlinks=true,
	linkcolor=black,
	filecolor=black,      
	urlcolor=black,
}

\newtheorem{proposition}{Proposition}

\newtheorem{corollary}{Corollary}
\newtheorem{definition}{Definition}

\title{Adversarial Moment-Matching Distillation of Large Language Models}

\author{%
  Chen Jia \\
  SI-TECH Information Technology \\
  \texttt{jiachenwestlake@gmail.com} \\
}

\begin{document}

\maketitle

\begin{abstract}
Knowledge distillation (KD) has been shown to be highly effective in guiding a student model with a larger teacher model and achieving practical benefits in improving the computational and memory efficiency for large language models (LLMs). State-of-the-art KD methods for LLMs mostly rely on minimizing explicit distribution distance between teacher and student probability predictions. Instead of optimizing these mandatory behaviour cloning objectives, we explore an imitation learning strategy for KD of LLMs. In particular, we minimize the imitation gap by matching the action-value moments of the teacher's behavior from both on- and off-policy perspectives. To achieve this action-value moment-matching goal, we propose an adversarial training algorithm to jointly estimate the moment-matching distance and optimize the student policy to minimize it. Results from both task-agnostic instruction-following experiments and task-specific experiments demonstrate the effectiveness of our method and achieve new state-of-the-art performance.
\end{abstract}

\section{Introduction}
Large language models (LLMs) like GPT-4 \cite{achiam2023gpt} and LLaMA \cite{touvron2023llama} have revolutionized natural language processing, significantly enhancing the quality of text generation across various tasks. This success is largely due to the extensive scale of training data and the substantial increase in model parameters \cite{kaplan2020scaling}. However, the high computational and memory requirements of these models present significant challenges for practical deployment. To address these issues, knowledge distillation (KD) \cite{hinton2015distilling} has emerged as a key technique. KD involves transferring knowledge from a large, complex teacher model to a smaller, more efficient student model, thereby maintaining high performance while reducing resource demands. Most distillation methods for auto-regressive text generation models, including LLMs, employ metrics of probability distribution distance, such as Kullback-Leibler (KL) divergence \cite{kim2016sequence} and reverse KL divergence \cite{gu2023minillm}, aiming to align the token-level probability distributions between the teacher and student models.

The distribution matching-based distillation methods can be viewed as behavior cloning on a decision-making problem from the perspective of imitation learning \cite{lin2020autoregressive,gu2023minillm,agarwal2024policy}. Based on this concept, early works based on the teacher-generated outputs \cite{kim2016sequence} or a supervised dataset \cite{sanh2019distilbert} can be viewed as an {\it off-policy} approach. Recent works further incorporate an {\it on-policy} approach, training the student on its self-generated outputs \cite{lin2020autoregressive}, using KL-based divergence \cite{gu2023minillm,agarwal2024policy,ko2024distillm} and total variation (TV) distance \cite{wen2023f}. Accordingly, such distribution matching-based methods face the sub-optimality problem. The objective functions aimed at aligning the probability distributions between the teacher and student models can be straightforward but cannot fully capture the goal of distilling language knowledge. First, intuitively, the correct output for an input can vary, and thus behavior cloning cannot capture the full knowledge of a teacher. Besides, there is no standardized definition for the quality of a generated output given an input, which makes it difficult to define the objective of knowledge distillation. This imposes a significant limitation on the generalization performance of the student model through distillation.

 \begin{figure}[t!]
 	\centering	
 	\includegraphics[width=0.95\linewidth]{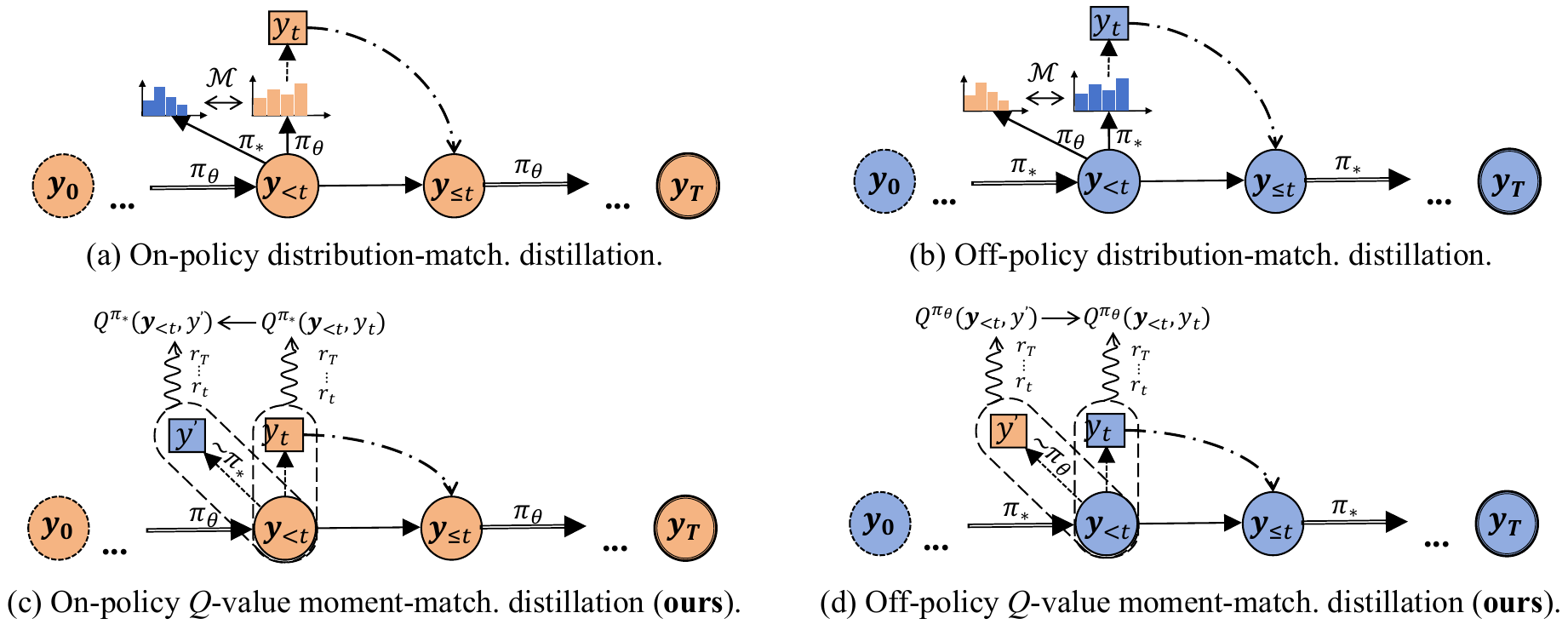}
 	\caption{The comparison between the distribution-matching-based distillation and the action-value moment-matching distillation is outlined. $\pi_{\theta}$ and $\pi_{*}$ denote the student policy and the teacher policy, respectively. For both on-policy (using student-generated outputs) and off-policy (using teacher-generated outputs) perspectives, our approach optimizes moment-matching of action-value functions ($Q$-functions) instead of minimizing the distribution distance measured by $\mathcal{M}$ = KL, RKL, TV, etc.}
 	\label{fig:overall}
 \end{figure}

To address the aforementioned issues, we employ a reinforcement learning (RL) formulation for the auto-regressive text generation problem and utilize the definition of imitation gap to describe the high-level goal of knowledge distillation. Additionally, we address the imitation gap for KD by matching moments of the action-value function, which reflects the quality of token-level predictions for the entire output. In addressing the action-value function, we adopt the approach of Swamy et al. \cite{swamy2021moments}, considering a two-player minimax game between the language policy and the action-value functions, aiming to minimize an upper bound of the moment-matching objective. For this purpose, we introduce an adversarial training algorithm based on the policy gradient to jointly optimize the on-/off-policy objectives. Figure \ref{fig:overall} illustrates the overall approach.

We evaluate our approach on both the instruction-following dataset and three task-specific datasets for text summarization, machine translation, and commonsense reasoning. Results demonstrate that the proposed adversarial moment-matching approach effectively optimizes the moment-matching distance of the imitation gap and outperforms state-of-the-art KD methods and a range of distribution-matching-based methods. The code and implementation is released at \url{https://github.com/jiachenwestlake/MMKD}.

\section{Related Work}
\noindent\textbf{Distillation of large language models.}
There has been an increasing interest in knowledge distillation (KD) of auto-regressive LMs, especially concerning large language models (LLMs) \cite{wu2024rethinking,xu2024survey}. This process effectively transfers elicited knowledge from teacher LLMs to smaller student models, aiming to compress the large size of neural network parameters and make LLMs more efficient. Sequence-level KD (SeqKD) \cite{kim2016sequence} is a variation of supervised fine-tuning (SFT) in KD. It can be viewed as the simplest method for distillation of black-box LLMs by fine-tuning the student model with teacher-generated outputs. This method has been extensively used for LLMs and has achieved success \cite{taori2023stanford,chiang2023vicuna}. In contrast, distillation of white-box LLMs can make full use of internal information of the teacher model, such as logits \cite{sanh2019distilbert,wen2023f} and hidden states \cite{liang2023less}, for distribution alignment, making it more effective and efficient for KD. However, unlike previous work that explicitly clones the distribution of teacher LLMs into student models, this work learns an auxiliary $Q$-value function to guide KD.

\noindent\textbf{Distillation via distribution matching.}
Most promising results in the distillation of white-box LLMs are achieved by minimizing divergence between the probability distributions of the teacher model and student models. Kullback-Leibler (KL) divergence, reverse Kullback-Leibler (RKL) divergence, and Jensen–Shannon (JS) divergence are three widely used KD objectives for auto-regressive LMs \cite{wen2023f,gu2023minillm,agarwal2024policy,ko2024distillm,wu2024rethinking}. Wen et al. \cite{wen2023f} have shown the equivalent formulations of sequence-level KL, RKL, JS divergences, and the step-wise terms. Additionally, they also present the strong performance of step-wise total variation (TV) distance for KD, which can upper bound the sequence-level term. As a result, most recent works focus on on-policy approaches for KD \cite{agarwal2024policy} and combine the real-time-generated outputs by students (on-policy) with the real-time-generated outputs by teachers (or from supervised datasets) (off-policy). Following this line, Gu et al. \cite{gu2023minillm} further propose a policy gradient-based method to address the high variance issues of RKL-based methods, while Ko et al. \cite{ko2024distillm} propose a more efficient and effective method using a skew KL divergence loss and an adaptive off-policy approach. We also focus on a combination of on-policy and off-policy objectives for KD, but we introduce a more sophisticated moment-matching approach instead of directly using the well-studied distribution-matching metrics such as KL, RKL, JS divergences, and TV distance.

\noindent\textbf{Distillation via reinforcement learning.}
In a common formulation of RL in text generation \cite{yu2017seqgan,pang2021text,hao2022teacher}, an auto-regressive model can be viewed as a language policy, making decisions on the next token (action) based on the currently generated sequence (state). From this perspective, KD corresponds to behavior cloning in imitation learning \cite{kim2016sequence,ciosek2021imitation,gu2023minillm,agarwal2024policy}. For imitation learning in text generation, early works such as SeqGAN \cite{yu2017seqgan} and TextGAIL \cite{wu2021textgail} utilize a generative adversarial framework to balance between the reward model, optimized by discriminating generated/real-word text, and the language policy, optimized by policy gradient-based methods using the reward model. Existing work on KD via imitation learning refers to ImitKD \cite{lin2020autoregressive}, which optimizes the student policy by learning from demonstrations of the teacher model. RL-based distillation can also be especially relevant for leveraging the feedback from the teacher to train student models \cite{bai2022constitutional,cui2023ultrafeedback}, in which teacher models are used to generate the feedback data for training a reward model. We build our method upon a RL-based imitation learning framework. However, unlike previous work \cite{kim2016sequence,gu2023minillm,agarwal2024policy}, we propose an adversarial moment-matching approach to enhance behavior cloning.

\section{Method}
\subsection{Notations and Definitions} \label{sec:notations}
In this section, we consider the text generation task as a decision-making process and give a corresponding reinforcement learning (RL) formulation.

\noindent\textbf{Text generation.} 
Given an input $\boldsymbol{x}$, the auto-regressive generation task in our work aims to generate a sequence of tokens as the output $(y_1, \ldots, y_T)$, where $y_t$ comes from a vocabulary $\mathcal{V}$. For simplicity, we define $\boldsymbol{y}=(y_0,y_1,\ldots,y_T)$ as the full input-output sequence, where $y_0 = \boldsymbol{x}$. The generator is modeled by a conditional probability distribution $p_{\theta}(\boldsymbol{y}|\boldsymbol{x}) = \Pi_{t=1}^T p_{\theta}(y_t|\boldsymbol{y}_{<t})$, where $\boldsymbol{y}_{<t}$ denotes the prefix $(y_0, y_1, \ldots, y_{t-1})$. 

\noindent\textbf{RL formulation.} 
We formulate text generation as a sequential decision-making process. At each time step $t \in \{1, \ldots, T\}$, the policy $\pi_{\theta}$ takes an action $y_t \in \mathcal{V}$ based on the current state $\boldsymbol{y}_{<t} \in {\mathcal{Y}}$, transits to the next state $\boldsymbol{y}_{< t+1} \in {\mathcal{Y}}$\footnote{In text generation, the state-transition is commonly assumed to be deterministic \cite{yu2017seqgan,pang2021text}, i.e., $p(\boldsymbol{y}_{< t+1}| \boldsymbol{y}_{<t},y_t)=1$.} and receives a reward $r(\boldsymbol{y}_{<t},y_t)$ by a reward function $r: {\mathcal{Y}} \times \mathcal{V} \rightarrow \mathbb{R}$. The policy corresponds to the generation model $\pi_{\theta}(y_t| \boldsymbol{y}_{<t}) = p_{\theta}(y_t|\boldsymbol{y}_{<t})$. A trajectory $\tau \sim \pi_{\theta} = \{\boldsymbol{y}_{<t}, y_t\}_{t=1,\ldots,T}$ refers to a sequence of state-action pairs generated by first sampling a state $y_0=\boldsymbol{x}$ from $p_{\boldsymbol{x}}$ and then repeatedly sampling an action
$y_t$ from $\pi_{\theta}(\cdot|\boldsymbol{y}_{<t})$ and obtain the next state $\boldsymbol{y}_{< t+1}$ for $T$
time steps. We also define our value function and $Q$-value function as
$V^{\pi_{\theta}}(\boldsymbol{y}_{<t}) = \mathbb{E}_{\tau \sim \pi_{\theta}|\boldsymbol{y}_{<t}} \sum_{t'=t}^T r(\boldsymbol{y}_{<t'},y_{t'})$ and $Q^{\pi_{\theta}}(\boldsymbol{y}_{<t},y_t) = \mathbb{E}_{\tau \sim \pi_{\theta}|\boldsymbol{y}_{<t},y_t} \sum_{t'=t}^T r(\boldsymbol{y}_{<t'},y_{t'})$.
We define the RL objective in our generation task to maximize the performance $J(\pi_\theta) = \mathbb{E}_{\tau \sim \pi_{\theta} } \left[ \sum_{t=1}^T r(\boldsymbol{y}_{<t}, y_t) \right]$. 
\iffalse
\noindent\textbf{Knowledge distillation.}
We are given two auto-regressive sequence models of different capacities, where $\pi_*$ and $\pi_{\theta}$ denote the teacher and $\theta$-parameterized student policy respectively. We are also given a dataset of inputs $\mathcal{D}_{\boldsymbol{x}}$. Optionally, we can also assume access to a dataset of input-output sequence pairs $\mathcal{D}_{\boldsymbol{x}\boldsymbol{y}}$. If not given, such a dataset can be generated by sampling sequences from the teacher policy given the input dataset $\mathcal{D}_{\boldsymbol{x}}$.
\fi
\subsection{Knowledge Distillation as Moment-Matching Imitation Learning}
Based on the RL formulation of auto-regressive generation, we can view the goal of knowledge distillation at a high-level as to bridge the performance gap between the teacher policy and the student policy.
\begin{definition}[\bf Imitation gap]
	We define the imitation gap between the teacher policy and student policy as:
\begin{align} \label{eq:imitgap}
	J(\pi_*) - J(\pi_{\theta}) = \mathbb{E}_{\tau \sim \pi_*} \left[ \sum_{t=1}^T r(\boldsymbol{y}_{<t}, y_t) \right] - \mathbb{E}_{\tau' \sim \pi_{\theta}} \left[ \sum_{t=1}^T r(\boldsymbol{y}'_{<t}, y'_t) \right],
\end{align}
where $\tau \sim \pi_* = \{\boldsymbol{y}_{<t}, y_t\}_{t=1,\ldots,T}$ denotes the trajectory of teacher policy and $\tau' \sim \pi_{\theta} = \{\boldsymbol{y}'_{<t}, y'_t\}_{t=1,\ldots,T}$ denotes the trajectory of student policy.
\end{definition}

From the perspective of imitation learning \cite{swamy2021moments,swamy2022sequence}, the objective of distillation from the teacher policy $\pi_{*}$ to the student policy $\pi_{\theta}$ can be represented as to minimize the imitation gap of Eq. (\ref{eq:imitgap}) w.r.t. the parameters of student policy $\theta$. An direct idea from Eq. (\ref{eq:imitgap}) is to use moment matching over the reward to optimize the imitation gap \cite{swamy2021moments}. However, we actually care about the long-term reward, at each time step, we should consider the accumulated reward in the future output rather than the immediate reward to the fitness of previous tokens (prefix). To this end, we can alternatively use the $Q$-value function (def. in Sec. \ref{sec:notations}) for each timestep to represent the overall reward from the current timestep to the last timestep. Similar to \cite{swamy2021moments}, we can apply the Performance Difference Lemma (PDL) to expand the imitation gap in Eq. (\ref{eq:imitgap}) into either off-policy or on-policy expressions.

\begin{proposition}[\bf Off-policy bound of imitation gap] \label{pro1}
	Let $\mathcal{F}_Q$ denote the set of $Q$-value functions induced by sampling actions from $\pi_{\theta}$, then we have:
	\begin{align} \label{eq:pro1}
		\begin{split}
			J(\pi_*) - J(\pi_{\theta}) \leq \sup_{f \in \mathcal{F}_{Q}} \mathop\mathbb{E}\limits_{\boldsymbol{x} \sim p_{\boldsymbol{x}} \atop \tau \sim \pi_*|\boldsymbol{y}_{0}=\boldsymbol{x}} \underbrace{\left[ \sum_{t=1}^T \left( f(\boldsymbol{y}_{<t}, y_t) - \mathop\mathbb{E}\limits_{y \sim \pi_{\theta}(\cdot| \boldsymbol{y}_{<t})} \left[ f(\boldsymbol{y}_{<t},y)\right] \right)\right] }_{\mathcal{U}^{\rm off}(\tau,f,\theta):=} 
		\end{split}
	\end{align}
	
	In the following sections, we will use $\mathcal{U}^{\rm off}(\tau,f,\theta)$ to represent a sampled off-policy imitation gap with an trajectory $\tau \sim \pi_*|\boldsymbol{y}_{0}=\boldsymbol{x}$ w.r.t. a $Q$-value function $f$ and a teacher policy $\pi_{\theta}$.
\end{proposition}

\begin{proof}
	The derivation is mostly followed Swamy et al. \cite{swamy2021moments}. See Appendix \ref{apdx:pro1} for the complete derivation.
\end{proof}

The off-policy upper bound of the imitation gap in Proposition \ref{pro1} only requires a collected dataset of teacher-generated trajectories to be evaluated and minimized.

\begin{proposition}[\bf On-policy bound of imitation gap] \label{pro2}
	Let $\mathcal{F}_{Q_*}$ denote the set of $Q$ functions induced by sampling actions from $\pi_*$, then we have:
	\begin{align} \label{eq:pro2}
		\begin{split}
			J(\pi_*) - J(\pi_{\theta}) \leq \sup_{f\in \mathcal{F}_{Q_*}} \mathop\mathbb{E}\limits_{\boldsymbol{x} \sim p_{\boldsymbol{x}} \atop \tau \sim	\pi_{\theta}|\boldsymbol{y}_{0}=\boldsymbol{x}}  \underbrace{\left[ \sum_{t=1}^T \left( \mathop\mathbb{E}\limits_{y\sim \pi_*(\cdot|\boldsymbol{y}_{<t})} \left[f(\boldsymbol{y}_{<t},y) \right] - f(\boldsymbol{y}_{<t}, y_t) \right)\right]}_{\mathcal{U}^{\rm on}(\tau,f):=}  
		\end{split}
	\end{align}
	In the following sections, we will use $\mathcal{U}^{\rm on}(\tau,f)$ to represent a sampled on-policy imitation gap with an trajectory $\tau \sim \pi_{\theta}|\boldsymbol{y}_{0}=\boldsymbol{x}$ w.r.t. a $Q$-value function $f$ given the teacher policy $\pi_{*}$.
\end{proposition}

\begin{proof}
	The derivation is mostly followed Swamy et al. \cite{swamy2021moments}. See Appendix \ref{apdx:pro2} for the complete derivation.
\end{proof}

It is notable from Proposition \ref{pro2} that the on-policy upper bound of the imitation gap requires interactions with the teacher to tell us what action they would take in any state visited by the student as well as on-policy samples from the student's current policy $\tau \sim \pi_{\theta}$.

In the remaining content of this section, we will exploring the relationship between the proposed upper bounds of the imitation gap and the existing distribution-matching objectives \cite{wen2023f}. At the beginning, we draw a general formulation of the state-of-the-art methods for distillation of LLMs \cite{wen2023f,gu2023minillm,agarwal2024policy,ko2024distillm} that rely on distribution-matching between the student's and teacher's predictions, through minimizing the step-wise probability distribution distance between the teacher policy and student policy. 
\begin{definition}[\bf Generalized step-wise distribution distance]\label{def:kdmetric}
	 The off-policy and on-policy versions are defined as follows,
	\begin{align}
		d_{\rm \mathcal{M}}^{\rm off}(\pi_{\theta},\pi_{*}) :=& \mathop\mathbb{E}\limits_{\boldsymbol{x} \sim p_{\boldsymbol{x}} \atop \tau \sim \pi_*|\boldsymbol{y}_{0}=\boldsymbol{x}}\!\!  \left[ \sum_{t=1}^T  {\mathcal{M}}(\pi_{*}(\cdot|\boldsymbol{y}_{<t}), \pi_{\theta}(\cdot|\boldsymbol{y}_{<t})) \right];\\
		d_{\mathcal{M}}^{\rm on}(\pi_{\theta},\pi_{*}) :=& \mathop\mathbb{E}\limits_{\boldsymbol{x} \sim p_{\boldsymbol{x}} \atop \tau \sim \pi_{\theta}|\boldsymbol{y}_{0}=\boldsymbol{x}}\!\!  \left[ \sum_{t=1}^T  {\mathcal{M}}(\pi_{*}(\cdot|\boldsymbol{y}_{<t}), \pi_{\theta}(\cdot|\boldsymbol{y}_{<t}))  \right],
    \end{align}
where $\mathcal{M}(\cdot,\cdot)$ denotes a distribution distance, consisting of total variation (TV) distance \cite{wen2023f} and Kullback-Leibler (KL)-based divergence \cite{gu2023minillm,agarwal2024policy}. Detailed definitions for these distances refer to Appendix \ref{apdx:kdmetric}.
\end{definition}

It is notable from Wen et al. \cite{wen2023f} that the sequence-level KL, RKL and JS divergences can be equivalently represented as the step-wise terms, and the sequence-level TV distance can be upper bounded by the step-wise terms, which can be actually implemented by algorithms. To make a connection with the step-wise distribution distance (Definition \ref{def:kdmetric}), we use the following definition.

\begin{definition}[\bf Distribution-matching formulation of moment-matching bounds] \label{def:dismm}
	Based on Definition \ref{def:kdmetric}, we can re-formulate the on-policy and off-policy moment-matching (MM) bounds (Proposition \ref{pro2} and Proposition \ref{pro1}, respectively) via step-wise distribution-matching, which can be defined as $d_{\mathcal{\rm MM}}^{\rm on}(\pi_{\theta},\pi_{*})$ and $d_{\rm MM}^{\rm off}(\pi_{\theta},\pi_{*})$ respectively, where the distance metric ${\rm MM}(\cdot, \cdot)$ can be defined as follows,
\begin{align}
\begin{split}
	{\rm MM}(\pi_{*}(\cdot|\boldsymbol{y}_{<t}), \pi_{\theta}(\cdot|\boldsymbol{y}_{<t})) &= \mathop\mathbb{E}\limits_{y\sim \pi_*(\cdot|\boldsymbol{y}_{<t})} \left[f_*(\boldsymbol{y}_{<t},y) \right] - \mathop\mathbb{E}\limits_{y' \sim \pi_{\theta}(\cdot| \boldsymbol{y}_{<t})} \left[ f_*(\boldsymbol{y}_{<t},y')\right], \\
	\text{\rm\bf Off-policy:}\ f_* &= \mathop{\arg\max}_{f\in \mathcal{F}_{Q_*}}{\mathop\mathbb{E}\limits_{\boldsymbol{x} \sim p_{\boldsymbol{x}} \atop \tau \sim	\pi_{*}|\boldsymbol{y}_{0}=\boldsymbol{x}}} \left[ \mathcal{U}^{\rm off}(\tau,f,\theta) \right]; \\
	\text{\rm\bf On-policy:}\ f_* &= \mathop{\arg\max}_{f\in \mathcal{F}_{Q}}{\mathop\mathbb{E}\limits_{\boldsymbol{x} \sim p_{\boldsymbol{x}} \atop \tau \sim	\pi_{\theta}|\boldsymbol{y}_{0}=\boldsymbol{x}}} \left[ \mathcal{U}^{\rm on}(\tau,f) \right]
\end{split}
\end{align}
\end{definition}

Under Definition \ref{def:dismm}, we observe that the main difference between the moment-matching bounds and other step-wise distribution distance, e.g., TV distance and KL-based divergences in formulation comes from the optimal $Q$-value function $f_{*}$, aiming to maximize the discrepancy of its expectations based on $\pi_*(\cdot|\boldsymbol{y}_{<t})$ {\it v.s.} $\pi_{\theta}(\cdot| \boldsymbol{y}_{<t})$ for each step $t \in \{1,2,\ldots,T\}$. To look deeper, we draw a connection between moment-matching bounds and the step-wise TV distance using the following corollary.
\begin{corollary} \label{cor1}
	Under a constrain on the class of $Q$-value functions: $\mathcal{F}_{Q_{\theta}} = \mathcal{F}_{Q_*} = \{f: \Vert f \Vert_{\infty} \leq 1\}$, the moment-matching bounds in Proposition \ref{pro1} and Proposition \ref{pro2} can be upper-bounded by the step-wise TV distance in Definition \ref{def:kdmetric}, Formally, for the off-policy and on-policy perspectives, we have:
	\begin{align}
	    J(\pi_*) - J(\pi_{\theta}) \leq& \sup_{f:\Vert f \Vert_{\infty} \leq 1} \mathop\mathbb{E}\limits_{\boldsymbol{x} \sim p_{\boldsymbol{x}} \atop \tau \sim \pi_*|\boldsymbol{y}_{0}=\boldsymbol{x}} \left[ \mathcal{U}^{\rm off}(\tau,f,\theta) \right]  \leq d_{\rm TV}^{\rm off}(\pi_{\theta},\pi_{*});\\
		J(\pi_*) - J(\pi_{\theta}) \leq& \sup_{f:\Vert f \Vert_{\infty} \leq 1} \mathop\mathbb{E}\limits_{\boldsymbol{x} \sim p_{\boldsymbol{x}} \atop \tau \sim	\pi_{\theta}|\boldsymbol{y}_{0}=\boldsymbol{x}} \left[ \mathcal{U}^{\rm on}(\tau,f) \right] \leq d_{\rm TV}^{\rm on}(\pi_{\theta},\pi_{*})
	\end{align} 
\end{corollary}
\begin{proof}
	See Appendix \ref{apdx:cor1} for the complete derivation.
\end{proof}

We can observe from Corollary \ref{cor1} that minimizing the step-wise TV distance can achieve sub-optimal results for the moment-matching bounds. Optimizing the moment-matching bounds can potentially achieve better optimization results for imitation learning objectives.

\subsection{Adversarial Training Algorithm}
\noindent\textbf{Optimization objective.}
\iffalse
Set the moment matching distance between the teacher policy $\pi_*$ and $\pi_{\theta}$ w.r.t. a trajectory $\tau = \{\boldsymbol{y}_{<t}, y_t\}_{t=1,\ldots,T}$ and a $\phi$-parameterized $Q$-value function $f_{\phi}: {{\mathcal{Y}}} \times \mathcal{V} \rightarrow \mathbb{R}$ as 
\begin{align}
	\mathcal{U}(\tau, \phi; \theta) = \sum_{t=1}^T \left( \mathop\mathbb{E}\limits_{y\sim \pi_*(\cdot|\boldsymbol{y}_{<t})}\!\! \left[f_{\phi}(\boldsymbol{y}_{<t},y) \right] - \mathop\mathbb{E}\limits_{y'\sim \pi_{\theta}(\cdot|\boldsymbol{y}_{<t})}\!\! \left[f_{\phi}(\boldsymbol{y}_{<t}, y') \right] \right)
\end{align} 
\fi
As shown in previous work \cite{gu2023minillm,agarwal2024policy,ko2024distillm} that incorporating both the on-policy and off-policy distillation benefits effectiveness and efficiency. We thus consider a training objective to jointly minimize the on-policy upper bound of Proposition (\ref{pro2}) and the off-policy upper bound of Proposition (\ref{pro1}). 
Both the moment-matching objectives can be optimized by viewing the learning procedure as solving a game. More specifically, we consider a two-player minimax game between the student policy and the $Q$-value functions. To this end, we consider using parameter estimation methods for the $Q$-value functions and assume that both $\mathcal{F}_{Q}$ and $\mathcal{F}_{Q_*}$ can be represented using a parameterized class of functions $\{f_\phi\}_{\phi\in\Phi}$, wherein we use $\phi_1$ and $\phi_2$ to denote parameters of the off-policy and on-policy $Q$-value functions, respectively. Thereby, the learning problem can be represented as follows,
\begin{align} \label{eq:obj}
    \min_{\theta \in \Theta} \max_{\phi_1,\phi_2 \in \Phi} \underbrace{\mathop\mathbb{E}\limits_{\boldsymbol{x} \sim p_{\boldsymbol{x}}} \left[ \mathop\mathbb{E}\limits_{\tau \sim \pi_*|\boldsymbol{y}_{0}=\boldsymbol{x}} \left[ \mathcal{U}^{\rm off}(\tau, f_{\phi_1}, \theta) \right] + \mathop\mathbb{E}\limits_{\tau' \sim	\pi_{\theta}|\boldsymbol{y}_{0}=\boldsymbol{x}} \left[ \mathcal{U}^{\rm on}(\tau', f_{\phi_2}) \right] \right]}_{\mathcal{L}(\theta,\phi_1,\phi_2):=}
\end{align} 

We denote the learning objective by $\mathcal{L}(\theta,\phi_1,\phi_2)$.
%\noindent\textbf{Policy gradient.}
To minimize the objective of $\mathcal{L}(\theta,\phi_1,\phi_2)$  w.r.t the policy parameters $\theta$, we use a policy gradient-based approach and derive the policy gradient of $\mathcal{L}(\theta)$ w.r.t $\theta$ in Appendix \ref{apdx:policygradient}, represented as follows,
\begin{align} \label{eq:gradpolicy}
	\begin{split}
    \nabla_{\theta} \mathcal{L}(\theta,\phi_1,\phi_2) &= \mathop\mathbb{E}\limits_{\boldsymbol{x} \sim p_{\boldsymbol{x}}} \bigg[ -\mathop\mathbb{E}\limits_{\tau \sim \pi_*|\boldsymbol{y}_{0}=\boldsymbol{x}}\left[ \mathcal{G}_{\rm off}(\tau, \theta) \right] +\mathop\mathbb{E}\limits_{\tau' \sim \pi_{\theta}|\boldsymbol{y}'_{0} = \boldsymbol{x}}\left[ \mathcal{G}_{\rm on}(\tau', \theta) \right] \bigg]\\
    {\rm s.t.}\quad \mathcal{G}_{\rm off}(\tau, \theta) &= \sum_{t=1}^T  \mathop\mathbb{E}\limits_{y'\sim \pi_{\theta}(\cdot|\boldsymbol{y}_{<t})} \nabla_{\theta}\log\pi_{\theta}(y'|\boldsymbol{y}_{<t}) f_{\phi_1}(\boldsymbol{y}_{<t}, y');\\
    \mathcal{G}_{\rm on}(\tau', \theta) &= \sum_{t=1}^T \nabla_{\theta}\log\pi_{\theta}(y'_t|\boldsymbol{y}'_{<t}) \mathcal{U}^{\rm on}(\tau', f_{\phi_2})
    \end{split}
\end{align}

Besides, to maximize the objective of $\mathcal{L}(\theta,\phi_1,\phi_2)$ w.r.t. parameters of the on-policy $Q$-value function $\phi_1$ and parameters of the off-policy $Q$-value function $\phi_2$, we use a stochastic gradient ascent method and represent the corresponding gradients as follows,
\begin{align} \label{eq:gradphi1}
\nabla_{\phi_1} \mathcal{L}(\theta,\phi_1,\phi_2) &= \mathop\mathbb{E}\limits_{\tau \sim \pi_*|\boldsymbol{y}_{0}=\boldsymbol{x}} \left[ \nabla_{\phi_1} \mathcal{U}^{\rm off}(\tau, f_{\phi_1}, \theta) \right] \\ \label{eq:gradphi2}
\nabla_{\phi_2} \mathcal{L}(\theta,\phi_1,\phi_2) &= \mathop\mathbb{E}\limits_{\tau' \sim	\pi_{\theta}|\boldsymbol{y}_{0}=\boldsymbol{x}} \left[ \nabla_{\phi_2} \mathcal{U}^{\rm on}(\tau', f_{\phi_2}) \right]
\end{align}

\RestyleAlgo{ruled}
%% This is needed if you want to add comments in
%% your algorithm with \Comment
\SetKwComment{Comment}{/* }{ */}

\begin{algorithm}[t!] 
	\small
	\caption{Adversarial training procedure} \label{alg:adv-train}
	\KwIn{Dataset $\mathcal{D}_{\boldsymbol{x}\boldsymbol{y}}$ with inputs and ground-truth outputs; \\
		Teacher policy $\pi_*$, student policy $\pi_{\theta}$ with initial parameters $\theta$ pretrained on $\mathcal{D}_{\boldsymbol{x}\boldsymbol{y}}$, off-policy $Q$-value function $f_{\phi_1}$, on-policy $Q$-value function $f_{\phi_2}$;\\ 
		Learning rate $\eta$, batch size $M$, control factor $\alpha$}
	\KwOut{A optimized student policy $\pi_{\theta}$}
	\While{$\theta$ has not converged}{
		\For{$k = 0,1,2, \ldots, K$}{
			Sample an input $\boldsymbol{x}\sim \mathcal{D}_{\boldsymbol{x}}$ and generate an output trajectory $\tau_{\rm off}\sim\pi_{*}|\boldsymbol{y}_{0} = \boldsymbol{x}$ \\
			$\phi_1 \gets \phi_1 + \alpha \eta \nabla_{\phi_1} \mathcal{U}^{\rm off}(\tau_{\rm off}, f_{\phi_1}, \theta)$ $\hfill \rhd$ maximize $\mathcal{L}(\theta,\phi_1,\phi_2)$ by Eq. (\ref{eq:gradphi1}) \\
			Sample an input $\boldsymbol{x}\sim \mathcal{D}_{\boldsymbol{x}}$ and generate an output trajectory $\tau_{\rm on}\sim\pi_{\theta}|\boldsymbol{y}_{0} = \boldsymbol{x}$ \\
			$\phi_2 \gets \phi_2 + \alpha \eta \nabla_{\phi_2} \mathcal{U}^{\rm on}(\tau_{\rm on}, f_{\phi_2}) $ $\hfill \rhd$ maximize $\mathcal{L}(\theta,\phi_1,\phi_2)$ by Eq. (\ref{eq:gradphi2}) \\
		}
		Sample an input $\boldsymbol{x}\sim \mathcal{D}_{\boldsymbol{x}}$ and generate output trajectories $\tau_{\rm off}\sim\pi_{*}|\boldsymbol{y}_{0} = \boldsymbol{x}$ and $\tau_{\rm on}\sim\pi_{\theta}|\boldsymbol{y}_{0} = \boldsymbol{x}$ \\
		$\theta \gets \theta - \eta\left( \mathcal{G}_{\rm on}(\tau_{\rm on}, \theta) - \mathcal{G}_{\rm off}(\tau_{\rm off}, \theta) \right)$ $\hfill \rhd$ minimize $\mathcal{L}(\theta,\phi_1,\phi_2)$ by Eq. (\ref{eq:gradpolicy})
	}
\end{algorithm}

\noindent\textbf{Training procedure.}
The training objective is to achieve a Nash equilibrium between the parameters of student policy $\theta$ and the parameters of on-policy and off-policy $Q$-value functions $\phi_1, \phi_2$ for the optimization problem of $\min_{\phi}\max_{\phi_1,\phi_2}\mathcal{L}(\theta,\phi_1,\phi_2)$ in Eq. (\ref{eq:obj}). To this end, we use an adversarial training strategy in Algorithm \ref{alg:adv-train}, by starting from a student model fine-tuned on a dataset $\mathcal{D}_{\boldsymbol{x}\boldsymbol{y}}$. In the training algorithm, we iteratively maximize the objective w.r.t. the parameters of $Q$-value functions $\phi_1$ $\phi_2$ and simultaneously minimize the objective w.r.t. the parameters of student policy $\theta$. In each iteration of policy updating, we first perform $K$ steps of gradient ascent w.r.t. the parameters of $Q$-value functions $\phi_1, \phi_2$, where the gradients in Eq. (\ref{eq:gradphi1}) and Eq. (\ref{eq:gradphi2}) are approximately computed by sampling inputs from the dataset and sampling the trajectories from the policies. Then, the parameters of student policy $\theta$ are updated by gradient descent with the estimated gradient of Eq. (\ref{eq:gradphi1}) with sampling inputs and policies.

\section{Experiments}
We consider task-agnostic instruction-following experiments and task-specific experiments, including text summarization, machine translation, and commonsense reasoning. We compare our approach with various KD baselines, including: SFT, which fine-tunes the student model on the supervised dataset $\mathcal{D}{\boldsymbol{x}\boldsymbol{y}}$; KD \cite{hinton2015distilling}, which uses KL divergence on the supervised dataset $\mathcal{D}{\boldsymbol{x}\boldsymbol{y}}$; SeqKD \cite{kim2016sequence}, which applies SFT to the student model with teacher-generated outputs; ImitKD \cite{lin2020autoregressive}, which uses KL divergence on the student-generated outputs; MiniLLM \cite{gu2023minillm}, which uses RKL divergence with a policy gradient method; GKD \cite{agarwal2024policy}, which uses JS divergence with an on-policy method; and DistiLLM \cite{ko2024distillm}, which uses an adaptive training method for off-policy optimization of a skew KL divergence. Additionally, we focus on step-wise distance optimization for KD and compare it with a range of well-known methods, including KL divergence, RKL divergence, JS divergence, and TV distance, as discussed by Wen et al. \cite{wen2023f}. All the reported results are
the average across three random seeds.

\subsection{Task-Agnostic Distillation}
\noindent\textbf{Experimental Setup.}
We follow the previous works \cite{gu2023minillm,ko2024distillm} for the implementation of the instruction-following experiment, aiming to evaluate the distilled model’s ability to handle diverse tasks presented in the form of instructions. We construct the training data from {\tt databricks-dolly-15k} \cite{dolly2023introducing}, where we randomly select 15K samples for training and equally split 500 samples for validation and testing. We evaluate the trained model on five instruction-following datasets: DollyEval, SelfInst \cite{wang2022self}, VicunaEval \cite{chiang2023vicuna}, S-NI \cite{wang2022super}, and UnNI \cite{honovich2022unnatural}. Following the previous works \cite{gu2023minillm,ko2024distillm}, we also add the OpenWebText \cite{gokaslan} corpus, consisting of long-document plain text, for joint training with a language modeling task. This has been shown to effectively improve the performance of instruction tuning \cite{gu2023minillm}. The evaluation metrics include ROUGE-L \cite{lin2004rouge} and GPT-4 feedback with the same prompts as in \cite{ko2024distillm}. More details on experimental setup refer to Appendix \ref{apdx:hyperparams}.

\begin{table*}[t!] 
	\centering
	\scalebox{.8}{
		\begin{tabular}{lcccccccc}
			\toprule
			\multirow{2}{*}{\bf Method}& \multicolumn{2}{c}{{DollyEval}} & \multicolumn{2}{c}{{SelfInst}} & \multicolumn{2}{c}{{VicunaEval}} & S-NI & UnNI \\
			\cmidrule(lr){2-3}\cmidrule(lr){4-5}\cmidrule(lr){6-7}\cmidrule(lr){8-8}\cmidrule(lr){9-9}
			& GPT4 & R-L & GPT4 & R-L & GPT4 & R-L & R-L & R-L \\
            \midrule
			{\it OpenLLaMA2-7B (teacher)} & {\it 58.8$_{\pm 1.2}$} & {\it 32.5$_{\pm 0.4}$} & {\it 56.7$_{\pm 0.8}$} & {\it 21.6$_{\pm 0.2}$} &{\it 46.2$_{\pm 0.6}$}  & {\it 22.6$_{\pm 0.5}$} & {\it 36.3$_{\pm 0.5}$} &{\it 38.5$_{\pm 0.2}$}  \\
			\midrule
			SFT (student) & 46.8$_{\pm 0.7}$ & 26.7$_{\pm 0.6}$ & 40.8$_{\pm 1.1}$ & 16.3$_{\pm 0.7}$ & 34.8$_{\pm 0.8}$ & 17.3$_{\pm 0.2}$& 30.4$_{\pm 0.4}$ &28.6$_{\pm 0.3}$ \\
			KD \cite{hinton2015distilling} & \textcolor{gray}{43.9}$_{\pm 0.8}$ & \textcolor{gray}{22.4}$_{\pm 0.4}$ & 43.5$_{\pm 0.5}$& 17.4$_{\pm 0.5}$ & \textcolor{gray}{33.7}$_{\pm 0.3}$ & \textcolor{gray}{16.4}$_{\pm 0.2}$ & \textcolor{gray}{29.3}$_{\pm 0.6}$ & \textcolor{gray}{23.4}$_{\pm 0.3}$ \\
			SeqKD \cite{kim2016sequence} & 50.2$_{\pm 0.6}$ & \textcolor{gray}{26.2}$_{\pm 0.4}$ & 46.8$_{\pm 0.3}$ & \textcolor{gray}{15.8}$_{\pm 0.5}$ & 38.8$_{\pm 1.2}$ & 18.0$_{\pm 0.6}$ & \textcolor{gray}{29.7}$_{\pm 0.3}$ & \textcolor{gray}{27.8}$_{\pm 0.1}$ \\
			ImitKD \cite{lin2020autoregressive} & 53.7$_{\pm 1.6}$ &  \textcolor{gray}{25.3}$_{\pm 0.3}$ & 45.0$_{\pm 0.7}$ & 18.4$_{\pm 0.4}$ & 41.7$_{\pm 1.2}$ & 19.1$_{\pm 0.2}$ & 33.1$_{\pm 0.7}$ & 28.7$_{\pm 0.5}$ \\
			MiniLLM \cite{gu2023minillm} & 58.7$_{\pm 1.2}$ & 28.4$_{\pm 0.3}$ & 51.8$_{\pm 1.5}$ & 20.2$_{\pm 0.6}$ & 44.2$_{\pm 1.1}$ & \underline{20.7}$_{\pm 0.5}$ & \underline{37.4}$_{\pm 0.4}$ & 37.5$_{\pm 0.2}$ \\
			GKD \cite{agarwal2024policy} & 57.6$_{\pm 1.0}$ & 27.5$_{\pm 0.3}$ & 52.4$_{\pm 1.2}$ & \underline{20.9}$_{\pm 0.3}$ & 45.5$_{\pm 0.8}$ & 19.3$_{\pm 0.5}$ & 36.8$_{\pm 0.6}$ & 34.8$_{\pm 0.3}$ \\
			DistiLLM \cite{ko2024distillm} & \underline{59.2}$_{\pm 1.2}$ & \underline{29.5}$_{\pm 0.2}$ & \underline{53.4}$_{\pm 1.0}$ &  20.8$_{\pm 0.7}$ & \underline{46.3}$_{\pm 0.9}$ &20.4$_{\pm 0.3}$ & 37.2$_{\pm 0.1}$ & \underline{38.2}$_{\pm 0.1}$\\
			{\bf ours} & {\bf 59.8}$_{\pm 0.8}$ & {\bf 30.7}$_{\pm 0.4}$ & {\bf 54.2}$_{\pm 1.2}$ &{\bf 21.7}$_{\pm 0.5}$ & {\bf 47.8}$_{\pm 0.7}$ & {\bf 21.4}$_{\pm 0.4}$ &{\bf 38.7}$_{\pm 0.4}$ & {\bf 39.1}$_{\pm 0.3}$ \\
			\bottomrule
	\end{tabular}}
	\caption{Comparison with state-of-the-art KD methods on the instruction-following dataset using fine-tuned OpenLLaMA-7B as the teacher and fine-tuned OpenLLaMA-3B as the student. We format {\bf the best}, \underline{the second best} and \textcolor{gray}{worse than SFT} results. The results based on GPT-2 are available in Appendix \ref{apdx:gpt2}.}
	\label{tbl:instfollow}
	 \vspace{-0.5cm}
\end{table*}

\noindent\textbf{Main results.}
Table \ref{tbl:instfollow} illustrates the instruction-following performances. Compared with the SFT baseline, which indicates the student model without KD, KD and SeqKD hardly improve the performances. This indicates that using only supervised datasets or teacher-generated outputs does not benefit the KD of large language models. In contrast, utilizing the student-generated outputs with KL divergence \cite{agarwal2024policy}, RKL divergence \cite{gu2023minillm}, and JS divergence \cite{agarwal2024policy} shows effectiveness for KD in the instruction-following task. State-of-the-art methods \cite{gu2023minillm,agarwal2024policy,ko2024distillm} tend to combine the student-generated outputs with the teacher-generated output or supervised dataset to further improve the results of KD. This shows that a mixture optimization of both on-policy and off-policy objectives can effectively improve the KD performance of large language models on the instruction-following task. In particular, we use an adversarial moment-matching method and optimize both on-policy and off-policy objectives for KD, thus achieving the best results on five test datasets with both GPT4 feedback and ROUGE-L evaluations.

\subsection{Task-Specific Distillation}
\noindent\textbf{Experimental Setup.}
We evaluated the KD models on three tasks consisting of text summarization, machine translation, and reasoning. For the text summarization task, we follow Ko et al. \cite{ko2024distillm} to conduct experiments on the SAMSum \cite{gliwa2019samsum} dataset. For the machine translation tasks, we follow Ko et al. \cite{ko2024distillm} to conduct experiments on the IWSLT'17 (en-de) \cite{cettolo2017overview} dataset. For the commonsense reasoning task, we conduct experiments on the StrategyQA dataset \cite{geva2021did}. For all of the task-specific experiments, we use T5-XL \cite{raffel2020exploring} as the teacher model and T5-Large/-Base/-Small as the student models. For the machine translation experiments, we employ a multilingual pretrained model, mT5 \cite{xue2020mt5}, to build the methods. For evaluation, we use ROUGE-L \cite{lin2004rouge}, BLEU \cite{papineni2002bleu}, and accuracy as the performance metrics on SAMSum, IWSLT'17 (en-de), and StrategyQA, respectively. More details about the experimental setup refer to Appendix \ref{apdx:hyperparams}.

\begin{wrapfigure}[23]{r}{0.65\textwidth}
	 \vspace{-0.5cm}
	\centering
	\subfigure[DollyEval.]{
		\begin{minipage}[h]{0.3\textwidth}
			\centering
			\includegraphics[width=4.5cm]{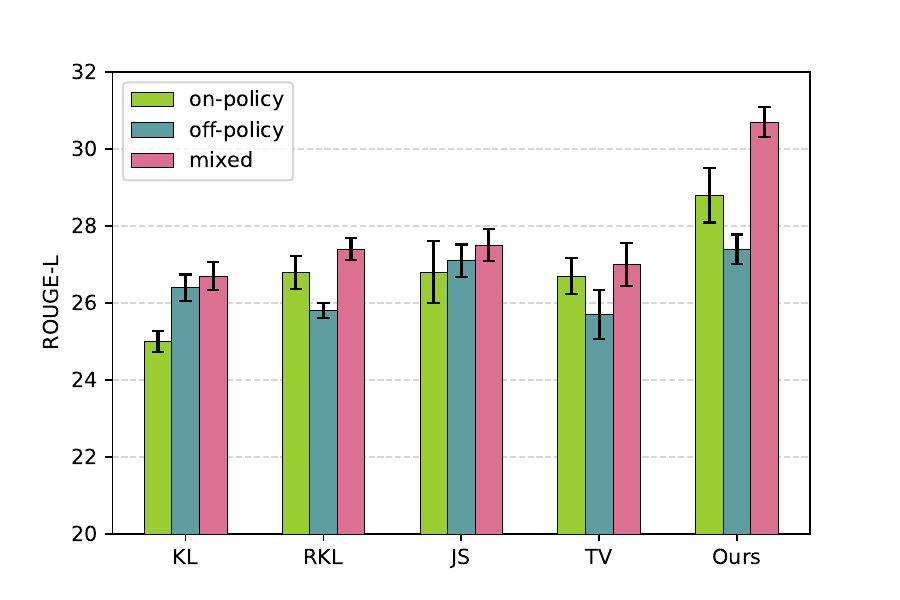}
	\end{minipage}}
	\subfigure[SAMSum.]{
		\begin{minipage}[h]{0.3\textwidth}
			\centering
			\includegraphics[width=4.5cm]{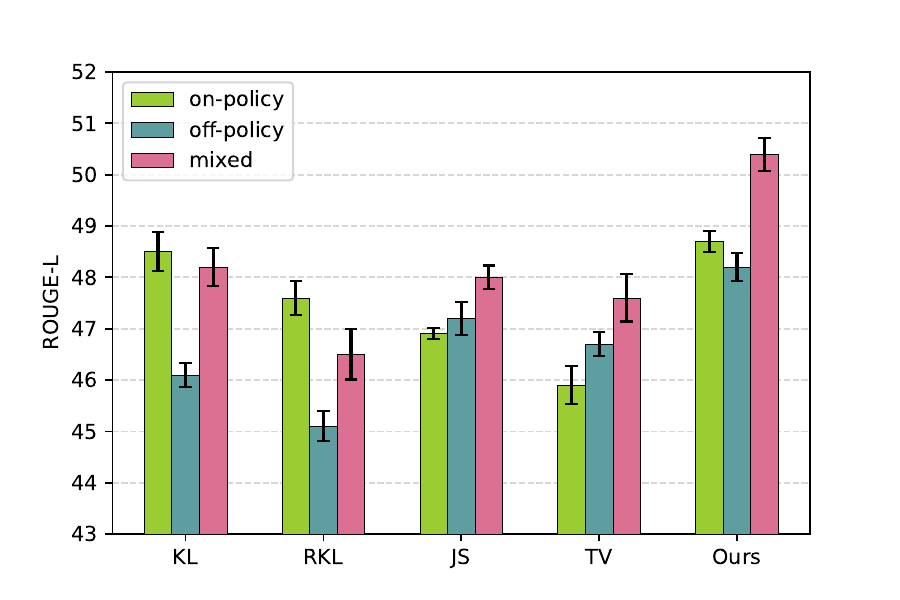}
	\end{minipage}}
	\subfigure[IWSLT'17 (en-de).]{
		\begin{minipage}[h]{0.3\textwidth}
			\centering
			\includegraphics[width=4.5cm]{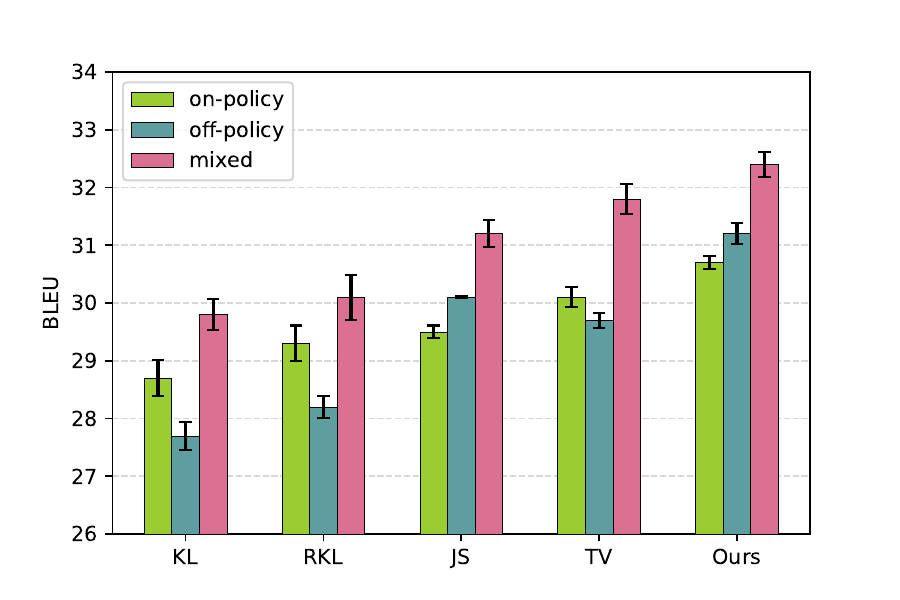}
	\end{minipage}}
	\subfigure[StrategyQA.]{
		\begin{minipage}[h]{0.3\textwidth}
			\centering
			\includegraphics[width=4.5cm]{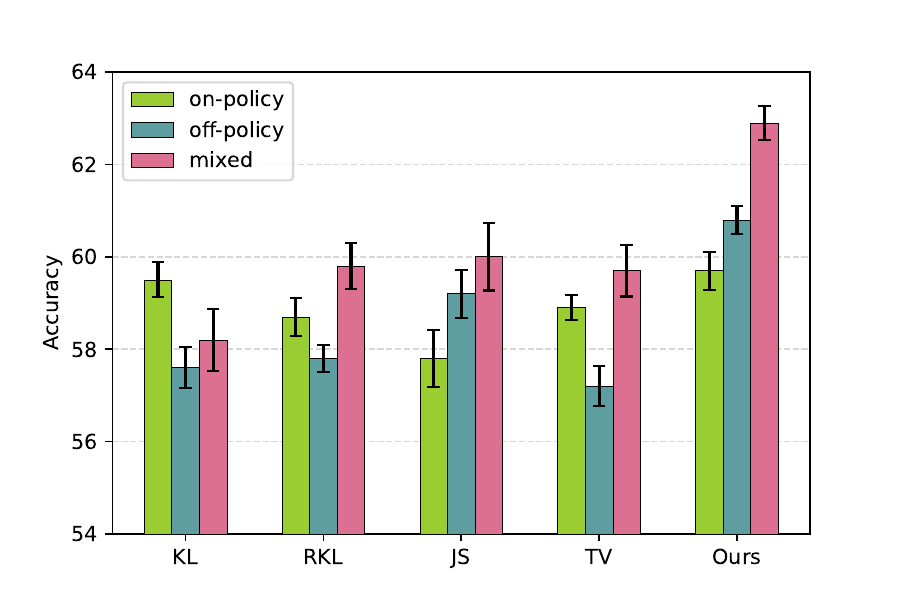}
	\end{minipage}}
	\caption{Performance of difference step-wise distribution distances.} \label{fig:performdist}
\end{wrapfigure}
\noindent\textbf{Main results.}
Table \ref{tbl:task} displays the performances on three task-specific datasets. Since the original work of MiniLLM \cite{gu2023minillm} does not consider these tasks, we thus do not make comparisons with MiniLLM. The performance trend is similar to the instruct-following results, revealing that KD of large language models for specific tasks also benefits from the combination of on-policy objectives with student-generated outputs and off-policy objectives with teacher-generated outputs or supervised datasets. Additionally, we observe that student models of different sizes all benefit from the KD methods to improve performance. Overall, our approach achieves the best results on all three task-specific datasets for student models of different sizes. This demonstrates the effectiveness of an adversarial moment-matching approach for KD of large language models on specific tasks.

\begin{table*}[t!] 
	\centering
	\scalebox{.75}{
		\begin{tabular}{lccccccccc}
			\toprule
			\multirow{2}{*}{\bf Method}& \multicolumn{3}{c}{{SAMSum}} & \multicolumn{3}{c}{{IWSLT'17 (en-de)}} & \multicolumn{3}{c}{{StrategyQA}} \\
			\cmidrule(lr){2-4}\cmidrule(lr){5-7}\cmidrule(lr){8-10}
			& T5-Small & T5-Base & T5-Large & T5-Small & T5-Base & T5-Large & T5-Small & T5-Base & T5-Large \\
			\midrule
			{\it T5-XL (teacher)} & \multicolumn{3}{c}{\it 52.5$_{\pm 0.4}$} & \multicolumn{3}{c}{\it 35.2$_{\pm 0.2}$}& \multicolumn{3}{c}{\it 64.5$_{\pm 0.8}$} \\
			\midrule
			SFT (student)  & 40.6$_{\pm 0.2}$ & 47.3$_{\pm 0.3}$ & 49.8$_{\pm 0.2}$ & 21.5$_{\pm 0.1}$ & 30.1$_{\pm 0.0}$ & 33.7$_{\pm 0.1}$ &52.4$_{\pm 0.5}$ &57.5$_{\pm 0.8}$ &60.7$_{\pm 0.8}$ \\
			KD \cite{hinton2015distilling}  & \textcolor{gray}{39.2}$_{\pm 0.4}$ & \textcolor{gray}{46.5}$_{\pm 0.3}$ & \textcolor{gray}{47.4}$_{\pm 0.3}$ & 21.7$_{\pm 0.1}$ & \textcolor{gray}{29.8}$_{\pm 0.2}$ & 31.7$_{\pm 0.1}$ & \textcolor{gray}{49.7}$_{\pm 0.3}$ & \textcolor{gray}{55.3}$_{\pm 0.1}$ & \textcolor{gray}{59.2}$_{\pm 0.5}$ \\
			SeqKD \cite{kim2016sequence}  & \textcolor{gray}{39.7}$_{\pm 0.3}$ & 47.7$_{\pm 0.5}$ & \textcolor{gray}{49.3}$_{\pm 0.4}$ & \textcolor{gray}{21.2}$_{\pm 0.3}$ &\textcolor{gray}{29.2}$_{\pm 0.2}$ & 32.9$_{\pm 0.5}$ & \textcolor{gray}{50.6}$_{\pm 0.7}$ & 57.5$_{\pm 1.1}$ & 61.5$_{\pm 0.8}$ \\
			ImitKD \cite{lin2020autoregressive}  & 41.8$_{\pm 0.3}$ & 48.6$_{\pm 0.7}$ & 51.2$_{\pm 0.5}$ & 22.2$_{\pm 0.3}$ & \textcolor{gray}{28.7}$_{\pm 0.6}$ &34.1$_{\pm 0.2}$ & 53.8$_{\pm 0.8}$ & 59.7$_{\pm 0.5}$ & 61.7$_{\pm 0.6}$  \\
			GKD \cite{agarwal2024policy}  & 42.1$_{\pm 0.3}$ & 48.2$_{\pm 0.5}$ & 51.7$_{\pm 0.4}$ &\underline{22.7}$_{\pm 0.2}$ & \underline{31.2}$_{\pm 0.1}$ & 34.7$_{\pm 0.2}$ & 55.6$_{\pm 0.4}$ & 60.3$_{\pm 0.5}$ & \underline{63.6}$_{\pm 0.3}$ \\
			DistiLLM \cite{ko2024distillm}  & \underline{42.6}$_{\pm 0.2}$ & \underline{49.4}$_{\pm 0.6}$ & \underline{52.1}$_{\pm 0.4}$ &22.5$_{\pm 0.1}$ & 30.8$_{\pm 0.2}$ & \underline{35.5}$_{\pm 0.1}$ & \underline{56.3}$_{\pm 0.3}$ &\underline{61.2}$_{\pm 0.7}$ & 62.8$_{\pm 0.2}$ \\
			{\bf ours}  & {\bf 43.7}$_{\pm 0.4}$ &{\bf 50.4}$_{\pm 0.3}$ & {\bf 52.7}$_{\pm 0.3}$ &{\bf 23.7}$_{\pm 0.1}$ &{\bf 32.4}$_{\pm 0.3}$ & {\bf 36.0}$_{\pm 0.2}$ & {\bf 58.2}$_{\pm 0.4}$ &{\bf 62.9}$_{\pm 0.3}$ & {\bf 65.3}$_{\pm 0.7}$ \\
			\bottomrule
	\end{tabular}}
	\caption{Comparison with the state-of-the-art KD methods on text summarization, machine	translation and commonsense reasoning datasets. We report the ROUGE-L, BLEU and accuracy for SAMSum, IWSLT'17 (en-de) and StrategyQA, respectively. We format {\bf the best}, \underline{the second best} and \textcolor{gray}{worse than SFT} results.}
	\label{tbl:task}
	 \vspace{-0.5cm}
\end{table*}

\subsection{Analysis on Step-Wise Distance Optimization}
\noindent\textbf{Comparison with distribution matching.}
We make comparisons with different step-wise distribution distances with a uniform formulation of Definition \ref{def:kdmetric}, considering the on-policy, off-policy objectives as well as the joint form. Results on four tasks are shown in Figure \ref{fig:performdist}, more instruct-following results are available in Appendix \ref{apdx:performdist}. Compared with the KL divergence, RKL divergence, JS divergence and total variation distance, the proposed moment-matching distance achieves the best results under both the on-policy and off-policy training objectives, which shows that the proposed moment-matching approach is effective for KD of large language models. Besides, we observe that using a joint objective of both on-policy and off-policy can further significantly improve the performances. This shows that both on-policy and off-policy moment-matching objectives contribute to the minimization of the imitation gap and can thus benefit the KD of large language models.

\begin{figure}[h]
	\vspace{-0.5cm}
	\centering
	\subfigure[Training loss and $d^{\rm on}_{\rm MM}$, $d^{\rm off}_{\rm MM}$ against training step.]{
		\begin{minipage}[h]{0.3\textwidth}
			\centering
			\includegraphics[width=4.5cm]{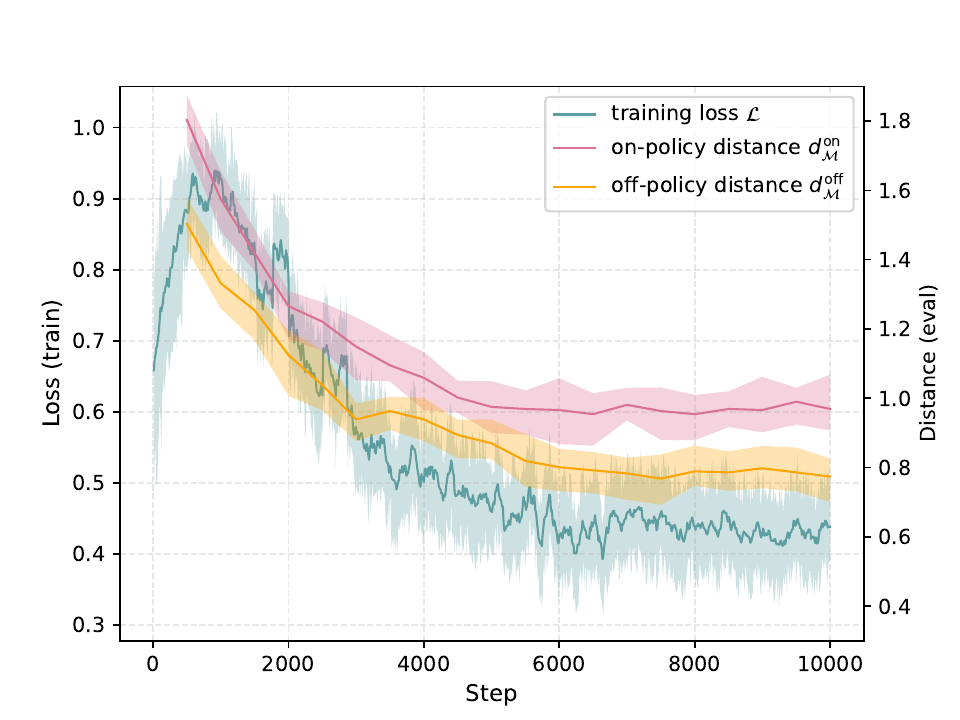}
	\end{minipage}}
    \hspace{0.1in}
	\subfigure[On-policy moment-matching distance $d^{\rm on}_{\rm MM}$ on the test sets.]{
		\begin{minipage}[h]{0.3\textwidth}
			\centering
			\includegraphics[width=4.5cm]{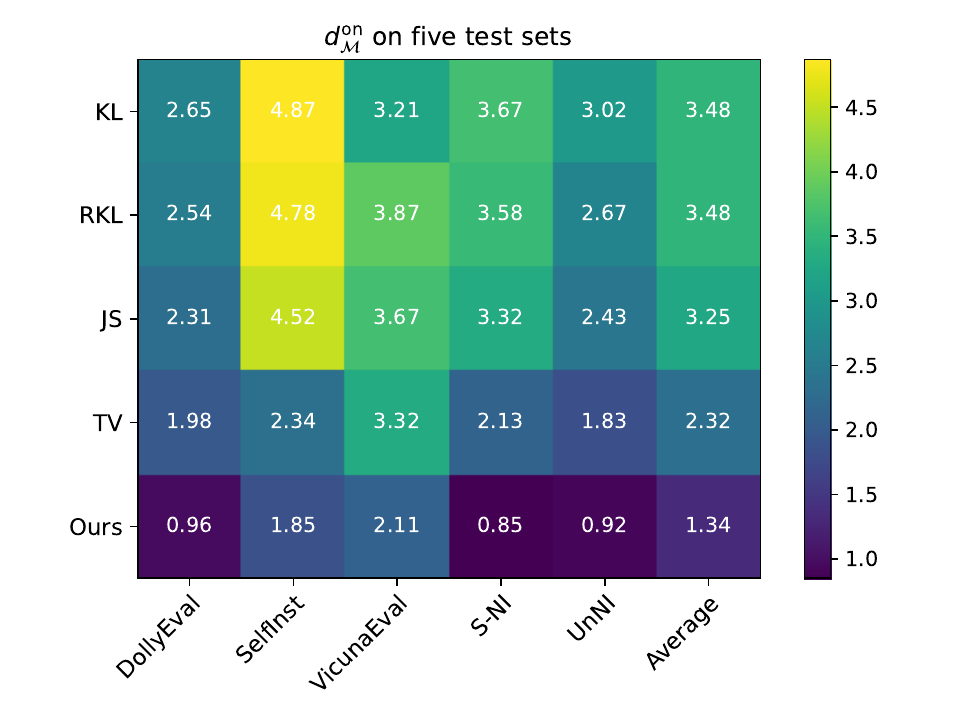}
	\end{minipage}}
\hspace{0.05in}
	\subfigure[Off-policy moment-matching distance $d^{\rm on}_{\rm MM}$ on the test sets.]{
		\begin{minipage}[h]{0.3\textwidth}
			\centering
			\includegraphics[width=4.5cm]{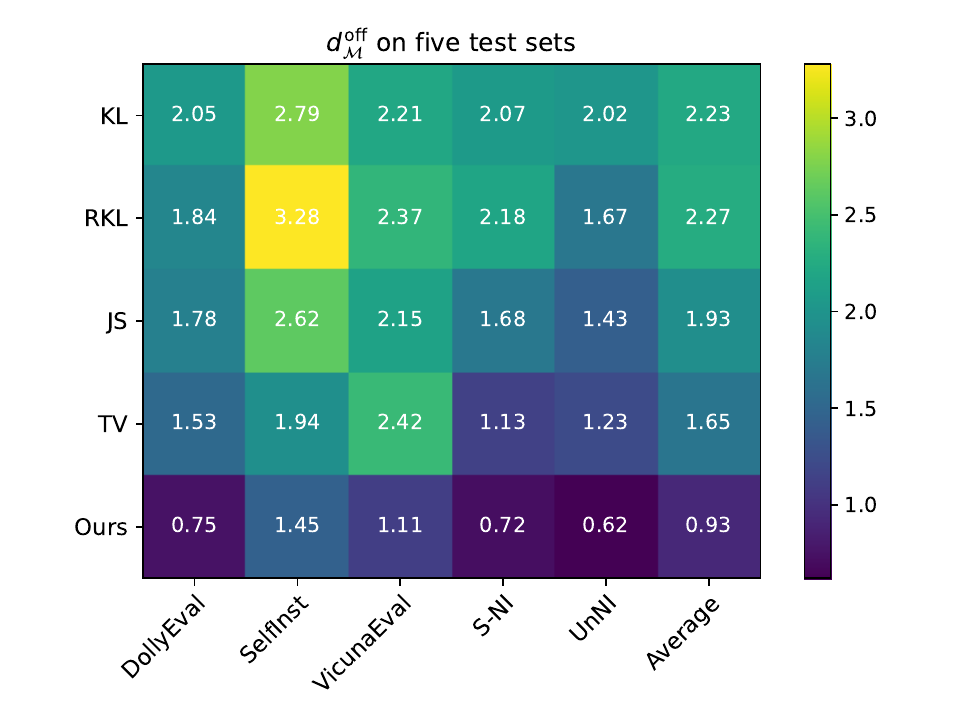}
	\end{minipage}} 
	\caption{Adversarial training procedure for optimizing the on-policy and off-policy moment-matching distances $d^{\rm on}_{\rm MM}$, $d^{\rm off}_{\rm MM}$ on the instruction-following dataset.} \label{fig:dist}
\end{figure}

\noindent\textbf{Adversarial training procedure.}
We present the training loss and moment-matching distance against the adversarial training steps. As depicted in Figure \ref{fig:dist} (a), the training loss initially increases within the first 0-1,000 steps, indicating that initially, the $Q$-value functions are stronger than the policy in maximizing the loss function $\mathcal{L}(\theta,\phi_1,\phi_2)$ in Eq. (\ref{eq:obj}). Concurrently, the policy gradient method contributes to minimizing the training loss, which eventually converges to a much lower stable value. Additionally, both the on-policy and off-policy moment-matching distances $d^{\rm on}{\rm MM}$ and $d^{\rm off}{\rm MM}$ decrease and eventually reach a low value with only minor fluctuations.
%This observation demonstrates that the proposed adversarial training algorithm effectively optimizes the training loss using the policy gradient method, while also effectively optimizing both the on-policy and off-policy moment-matching distances. 
For more results and details on experimental setups, please refer to Appendix \ref{apdx:loss}.

\noindent\textbf{Moment-matching distance optimization.}
We further illustrate the on-policy moment-matching distance $d^{\rm on}{\rm MM}$ and the off-policy moment-matching distance $d^{\rm off}{\rm MM}$ (defined in Definition \ref{def:dismm}) optimized by different step-wise distances in Figure \ref{fig:dist} (b) and (c), respectively. Interestingly, we observe that the total variation (TV) distance obtains the second-best results on average for both on-policy and off-policy distances. This finding suggests a similarity between the formulations of TV distance and moment-matching distances to some extent, as supported by the theoretical result of Corollary \ref{cor1}. Across all instruction-following test sets, our approach effectively optimizes both on-policy and off-policy moment-matching distances more than other step-wise distribution distances used in KD, including KL divergence, RKL divergence, JS divergence, and TV distance. 
This observation also underscores the effectiveness of our policy gradient methods. Extensive results on the task-specific datasets are available in Appendix \ref{apdx:dist}.

\section{Conclusion}
In this work, we investigated a moment-matching approach for knowledge distillation of large language models. Specifically, we formulated knowledge distillation from a perspective of imitation learning and derived both on-policy and off-policy bounds for the imitation gap between the teacher model and student model via moment-matching distance. Additionally, we proposed an adversarial training algorithm to simultaneously estimate and minimize the joint objective of on-policy and off-policy moment-matching distances. In experiments, we evaluated the proposed algorithm on four instruction-following datasets and three task-specific datasets, comparing it with a range of state-of-the-art KD methods as well as four well-studied step-wise distribution distances for KD of auto-regressive models. Results demonstrate that our approach can effectively leverage the policy gradient method to optimize the moment-matching distance and achieve the best results across all datasets.

\noindent\textbf{Limitations and future work.} 
The proposed adversarial training algorithm requires additional computational steps for the inner-loop gradient ascent, which may result in increased time complexity. Moreover, the proposed approach necessitates auxiliary networks to build the $Q$-value functions, which may incur additional memory costs. Therefore, in future work, we aim to enhance the time and memory efficiency of our approach.

\bibliographystyle{plain}
\bibliography{neurips_2024}

\newpage

\newpage
\appendix

\section{Proofs} \label{proofs}
\subsection{Proof of Proposition \ref{pro1}} \label{apdx:pro1}
\begin{proof}
Similar to the proof of Performance Difference Lemma (PDL), we have
\begin{align*}
	&J(\pi_*) - J(\pi_{\theta}) \\
	=& \mathop\mathbb{E}\limits_{\boldsymbol{x} \sim p_{\boldsymbol{x}} \atop \tau \sim \pi_*|\boldsymbol{y}_{0}=\boldsymbol{x}}\left[ \sum_{t=1}^T r(\boldsymbol{y}_{<t}, y_t) \right] - \mathop\mathbb{E}\limits_{\boldsymbol{x} \sim p_{\boldsymbol{x}}} \left[ V^{\pi_{\theta}} (\boldsymbol{x}) \right] \\
	=& \mathop\mathbb{E}\limits_{\boldsymbol{x} \sim p_{\boldsymbol{x}} \atop \tau \sim \pi_*|\boldsymbol{y}_{0}=\boldsymbol{x}}\left[ \sum_{t=1}^T \left( r(\boldsymbol{y}_{<t}, y_t) + V^{\pi_{\theta}}(\boldsymbol{y}_{<t}) - V^{\pi_{\theta}}(\boldsymbol{y}_{<t}) \right) \right] - \mathop\mathbb{E}\limits_{\boldsymbol{x} \sim p_{\boldsymbol{x}}} \left[ V^{\pi_{\theta}} (\boldsymbol{x}) \right] \\
	=& \mathop\mathbb{E}\limits_{\boldsymbol{x} \sim p_{\boldsymbol{x}} \atop \tau \sim \pi_*|\boldsymbol{y}_{0}=\boldsymbol{x}} \left[ \sum_{t=1}^T \left( r(\boldsymbol{y}_{<t}, y_t) + V^{\pi_{\theta}}(\boldsymbol{y}_{< t+1}) - V^{\pi_{\theta}}(\boldsymbol{y}_{<t})  \right) \right] \\
	\overset{(i)}{=}& \mathop\mathbb{E}\limits_{\boldsymbol{x} \sim p_{\boldsymbol{x}} \atop \tau \sim \pi_*|\boldsymbol{y}_{0}=\boldsymbol{x}} \left[ \sum_{t=1}^T \left( Q^{\pi_{\theta}}(\boldsymbol{y}_{<t}, y_t) -  V^{\pi_{\theta}}(\boldsymbol{y}_{<t}) \right)  \right]\\
	=& \mathop\mathbb{E}\limits_{\boldsymbol{x} \sim p_{\boldsymbol{x}} \atop \tau \sim \pi_*|\boldsymbol{y}_{0}=\boldsymbol{x}} \left[ \sum_{t=1}^T \left( Q^{\pi_{\theta}}(\boldsymbol{y}_{<t}, y_t) - \mathop\mathbb{E}\limits_{y \sim \pi_{\theta}(\cdot| \boldsymbol{y}_{<t})} \left[ Q^{\pi_{\theta}}(\boldsymbol{y}_{<t},y)\right] \right) \right] \\
	%=& \mathop\mathbb{E}\limits_{\boldsymbol{x} \sim p_{\boldsymbol{x}} \atop \tau \sim \pi_*|\boldsymbol{y}_{0}=\boldsymbol{x}} \left[ \sum_{t=1}^T \left( \mathop\mathbb{E}\limits_{y\sim \pi_*(\cdot| \boldsymbol{y}_{<t})} \left[Q^{\pi_{\theta}}(\boldsymbol{y}_{<t}, y)\right] - \mathop\mathbb{E}\limits_{y' \sim \pi_{\theta}(\cdot| \boldsymbol{y}_{<t})} \left[ Q^{\pi_{\theta}}(\boldsymbol{y}_{<t},y')\right] \right) \right] \\
	%\leq& \sup_{f \in \mathcal{F}_{Q}} \mathop\mathbb{E}\limits_{\boldsymbol{x} \sim p_{\boldsymbol{x}} \atop \tau \sim \pi_*|\boldsymbol{y}_{0}=\boldsymbol{x}} \left[ \sum_{t=1}^T \left( \mathop\mathbb{E}\limits_{y\sim \pi_*(\cdot| \boldsymbol{y}_{<t})} \left[f(\boldsymbol{y}_{<t}, y)\right] - \mathop\mathbb{E}\limits_{y' \sim \pi_{\theta}(\cdot| \boldsymbol{y}_{<t})} \left[ f(\boldsymbol{y}_{<t},y')\right] \right) \right] 
	\leq& \sup_{f \in \mathcal{F}_{Q}} \mathop\mathbb{E}\limits_{\boldsymbol{x} \sim p_{\boldsymbol{x}} \atop \tau \sim \pi_*|\boldsymbol{y}_{0}=\boldsymbol{x}} \left[ \sum_{t=1}^T \left( f(\boldsymbol{y}_{<t}, y_t) - \mathop\mathbb{E}\limits_{y \sim \pi_{\theta}(\cdot| \boldsymbol{y}_{<t})} \left[ f(\boldsymbol{y}_{<t},y)\right] \right) \right],
\end{align*}
where $(i)$ follows from Bellman equation.
\end{proof}

\subsection{Proof of Proposition \ref{pro2}} \label{apdx:pro2}
\begin{proof}
	Similar to the proof of Proposition \ref{pro1}, we have
	\begin{align*}
		&J(\pi_*) - J(\pi_{\theta}) \\
		=& - \mathop\mathbb{E}\limits_{\boldsymbol{x} \sim p_{\boldsymbol{x}} \atop \tau \sim	\pi_{\theta}|\boldsymbol{y}_{0}=\boldsymbol{x}}\left[ \sum_{t=1}^T r(\boldsymbol{y}_{<t}, y_t) \right] + \mathop\mathbb{E}\limits_{\boldsymbol{x} \sim p_{\boldsymbol{x}}} \left[ V^{\pi_*} (\boldsymbol{x}) \right] \\
		=& \mathop\mathbb{E}\limits_{\boldsymbol{x} \sim p_{\boldsymbol{x}} \atop \tau \sim	\pi_{\theta}|\boldsymbol{y}_{0}=\boldsymbol{x}}\left[ \sum_{t=1}^T \left( V^{\pi_*}(\boldsymbol{y}_{<t}) - \left( r(\boldsymbol{y}_{<t}, y_t) + V^{\pi_*}(\boldsymbol{y}_{<t}) \right) \right) \right] + \mathop\mathbb{E}\limits_{\boldsymbol{x} \sim p_{\boldsymbol{x}}} \left[ V^{\pi_*} (\boldsymbol{x}) \right] \\
		=& \mathop\mathbb{E}\limits_{\boldsymbol{x} \sim p_{\boldsymbol{x}} \atop \tau \sim	\pi_{\theta}|\boldsymbol{y}_{0}=\boldsymbol{x}}\left[ \sum_{t=1}^T \left( V^{\pi_*}(\boldsymbol{y}_{<t}) - \left( r(\boldsymbol{y}_{<t}, y_t) + V^{\pi_*}(\boldsymbol{y}_{< t+1}) \right) \right) \right] \\
		=& \mathop\mathbb{E}\limits_{\boldsymbol{x} \sim p_{\boldsymbol{x}} \atop \tau \sim	\pi_{\theta}|\boldsymbol{y}_{0}=\boldsymbol{x}}\left[ \sum_{t=1}^T \left( V^{\pi_*}(\boldsymbol{y}_{<t}) - Q^{\pi_*}(\boldsymbol{y}_{<t}, y_t) \right) \right] \\
		=& \mathop\mathbb{E}\limits_{\boldsymbol{x} \sim p_{\boldsymbol{x}} \atop \tau \sim	\pi_{\theta}|\boldsymbol{y}_{0}=\boldsymbol{x}}\left[ \sum_{t=1}^T \left( \mathop\mathbb{E}\limits_{y\sim \pi_*(\cdot|\boldsymbol{y}_{<t})} \left[Q^{\pi_*}(\boldsymbol{y}_{<t},y) \right] - Q^{\pi_*}(\boldsymbol{y}_{< t+1}, y_t) \right) \right] \\
		%=& \mathop\mathbb{E}\limits_{\boldsymbol{x} \sim p_{\boldsymbol{x}} \atop \tau \sim	\pi_{\theta}|\boldsymbol{y}_{0}=\boldsymbol{x}}\left[ \sum_{t=1}^T \left( \mathop\mathbb{E}\limits_{y\sim \pi_*(\cdot|\boldsymbol{y}_{<t})} \left[Q^{\pi_*}(\boldsymbol{y}_{<t},y) \right] - \mathop\mathbb{E}\limits_{y'\sim \pi_{\theta}(\cdot|\boldsymbol{y}_{<t})} \left[Q^{\pi_*}(\boldsymbol{y}_{< t+1}, y') \right] \right)  \right] \\
		%\leq& \sup_{f\in \mathcal{F}_{\mathcal{F}_{Q_T}}} \mathop\mathbb{E}\limits_{\boldsymbol{x} \sim p_{\boldsymbol{x}} \atop \tau \sim	\pi_{\theta}|\boldsymbol{y}_{0}=\boldsymbol{x}}\left[ \sum_{t=1}^T \left( \mathop\mathbb{E}\limits_{y\sim \pi_*(\cdot|\boldsymbol{y}_{<t})} \left[f(\boldsymbol{y}_{<t},y) \right] - \mathop\mathbb{E}\limits_{y'\sim \pi_{\theta}(\cdot|\boldsymbol{y}_{<t})} \left[f(\boldsymbol{y}_{<t}, y') \right] \right)  \right]
		\leq& \sup_{f\in {\mathcal{F}_{Q_*}}} \mathop\mathbb{E}\limits_{\boldsymbol{x} \sim p_{\boldsymbol{x}} \atop \tau \sim	\pi_{\theta}|\boldsymbol{y}_{0}=\boldsymbol{x}}\left[ \sum_{t=1}^T \left( \mathop\mathbb{E}\limits_{y\sim \pi_*(\cdot|\boldsymbol{y}_{<t})} \left[f(\boldsymbol{y}_{<t},y) \right] - f(\boldsymbol{y}_{<t}, y_t) \right)  \right]
	\end{align*}
\end{proof}

\subsection{Existing Step-Wise Distribution Distance for KD} \label{apdx:kdmetric}
\begin{definition}[\bf Step-wise distribution distances for distillation] \label{def:stepdist}
	Following Wen et al. \cite{wen2023f}, we define four groups of well-studied probability distribution distances as follows,
	\begin{itemize}
		\item {\bf Total variation (TV) distance.} \cite{wen2023f} The off-policy and on-policy step-wise TV distance can be defined as follows,
		\begin{align}
		    d_{\rm TV}^{\rm off}(\pi_{\theta},\pi_{*}) :=& \mathop\mathbb{E}\limits_{\boldsymbol{x} \sim p_{\boldsymbol{x}} \atop \tau \sim \pi_*|\boldsymbol{y}_{0}=\boldsymbol{x}}\!\!  \left[ \sum_{t=1}^T \sum_{y\in \mathcal{V}} \left\vert \pi_*(y|\boldsymbol{y}_{<t}) - \pi_{\theta}(y|\boldsymbol{y}_{<t}) \right\vert \right]; \\
			d_{\rm TV}^{\rm on}(\pi_{\theta},\pi_{*}) :=& \mathop\mathbb{E}\limits_{\boldsymbol{x} \sim p_{\boldsymbol{x}} \atop \tau \sim \pi_{\theta}|\boldsymbol{y}_{0}=\boldsymbol{x}}\!\!  \left[ \sum_{t=1}^T \sum_{y\in \mathcal{V}} \left\vert \pi_*(y|\boldsymbol{y}_{<t}) - \pi_{\theta}(y|\boldsymbol{y}_{<t}) \right\vert \right]
		\end{align}
	    \iffalse
		\item {\bf Wasserstein (W) distance} By the Kantorovich-Rubinstein
		\begin{align}
			d_{\rm W}^{\rm on}(\pi_{\theta},\pi_{*}) :=& \mathop\mathbb{E}\limits_{\boldsymbol{x} \sim p_{\boldsymbol{x}} \atop \tau \sim \pi_{\theta}|\boldsymbol{y}_{0}=\boldsymbol{x}}\!\!  \left[ \sum_{t=1}^T \sup_{f:\Vert f \Vert_L \leq 1} \mathbb{E}_{y \sim } \right] \\
			d_{\rm W}^{\rm off}(\pi_{\theta},\pi_{*}) :=& \mathop\mathbb{E}\limits_{\boldsymbol{x} \sim p_{\boldsymbol{x}} \atop \tau \sim \pi_*|\boldsymbol{y}_{0}=\boldsymbol{x}}\!\!  \left[ \sum_{t=1}^T \inf_{\mu \in \mathcal{L}(\pi_{*}(\cdot|\boldsymbol{y}_{<t}), \pi_{\theta}(\cdot|\boldsymbol{y}_{<t}))}  \right]
		\end{align}
	    \fi
		\item {\bf Kullback–Leibler (KL) divergence} \cite{wen2023f} The off-policy and on-policy step-wise KL divergence between the teacher policy $\pi_{*}$ and the student policy $\pi_{\theta}$ can be defined as follows,
	    \begin{align}
	        d_{\rm KL}^{\rm off}(\pi_{*},\pi_{\theta}) :=& \mathop\mathbb{E}\limits_{\boldsymbol{x} \sim p_{\boldsymbol{x}} \atop \tau \sim \pi_*|\boldsymbol{y}_{0}=\boldsymbol{x}} \left[ \sum_{t=1}^T \sum_{y\in \mathcal{V}} \pi_{*}(y|\boldsymbol{y}_{<t}) \log \frac{\pi_{*}(y|\boldsymbol{y}_{<t})}{\pi_{\theta}(y|\boldsymbol{y}_{<t})} \right];\\
	    	d_{\rm KL}^{\rm on}(\pi_{*},\pi_{\theta}) :=&  \mathop\mathbb{E}\limits_{\boldsymbol{x} \sim p_{\boldsymbol{x}} \atop \tau \sim \pi_{\theta}|\boldsymbol{y}_{0}=\boldsymbol{x}} \left[ \sum_{t=1}^T \sum_{y\in \mathcal{V}} \pi_{*}(y|\boldsymbol{y}_{<t}) \log \frac{\pi_{*}(y|\boldsymbol{y}_{<t})}{\pi_{\theta}(y|\boldsymbol{y}_{<t})} \right]
        \end{align}
	    \item {\bf Reverse Kullback–Leibler (RKL) divergence} \cite{wen2023f} The off-policy and on-policy step-wise RKL divergence between the teacher policy $\pi_{*}$ and the student policy $\pi_{\theta}$ can be defined as follows,
	    \begin{align}
	        d_{\rm RKL}^{\rm off}(\pi_{\theta},\pi_{*}) :=& \mathop\mathbb{E}\limits_{\boldsymbol{x} \sim p_{\boldsymbol{x}} \atop \tau \sim \pi_*|\boldsymbol{y}_{0}=\boldsymbol{x}} \left[ \sum_{t=1}^T \sum_{y\in \mathcal{V}} \pi_{\theta}(y|\boldsymbol{y}_{<t}) \log \frac{\pi_{\theta}(y|\boldsymbol{y}_{<t})}{\pi_{*}(y|\boldsymbol{y}_{<t})} \right];\\
	    	d_{\rm RKL}^{\rm on}(\pi_{\theta},\pi_{*}) :=&  \mathop\mathbb{E}\limits_{\boldsymbol{x} \sim p_{\boldsymbol{x}} \atop \tau \sim \pi_{\theta}|\boldsymbol{y}_{0}=\boldsymbol{x}} \left[ \sum_{t=1}^T \sum_{y\in \mathcal{V}} \pi_{\theta}(y|\boldsymbol{y}_{<t}) \log \frac{\pi_{\theta}(y|\boldsymbol{y}_{<t})}{\pi_{*}(y|\boldsymbol{y}_{<t})} \right]
	    \end{align}
		\item {\bf Jenson–Shannon (JS) divergence} \cite{wen2023f} The JS divergence between the teacher policy $\pi_{*}$ and the student policy $\pi_{\theta}$ can be defined based on the KL divergence and RKL divergence as follows,
		\begin{align}
			d_{\rm JS}^{\rm off}(\pi_{\theta},\pi_{*}) :=& \frac{1}{2} d_{\rm KL}^{\rm off}(\pi_{*}, \frac{\pi_{\theta}+\pi_{*}}{2}) + \frac{1}{2} d_{\rm RKL}^{\rm off}(\pi_{\theta},\frac{\pi_{\theta}+\pi_{*}}{2}) ;\\
			d_{\rm JS}^{\rm on}(\pi_{\theta},\pi_{*}) :=& \frac{1}{2} d_{\rm KL}^{\rm on}(\pi_{*}, \frac{\pi_{\theta} + \pi_{*}}{2}) + \frac{1}{2} d_{\rm RKL}^{\rm on}(\pi_{\theta},\frac{\pi_{\theta}+\pi_{*}}{2})
		\end{align}
	\end{itemize}
\end{definition}

\subsection{Proof of Corollary \ref{cor1}} \label{apdx:cor1}
\begin{proof}
We first derive an upper bound for the on-policy moment-matching bound of Eq. (\ref{eq:pro2}). Set $\mathcal{F}_{Q_*} = \{f: \Vert f \Vert_{\infty} \leq 1\}$, Eq. (\ref{eq:pro2}) can be bounded as follows,
\begin{align*}
	&\sup_{f:\Vert f \Vert_{\infty} \leq 1} \mathop\mathbb{E}\limits_{\boldsymbol{x} \sim p_{\boldsymbol{x}} \atop \tau \sim	\pi_{\theta}|\boldsymbol{y}_{0}=\boldsymbol{x}} \left[ \mathcal{U}^{\rm on}(\tau,f) \right] \\
	=& \sup_{f:\Vert f \Vert_{\infty} \leq 1} \mathop\mathbb{E}\limits_{\boldsymbol{x} \sim p_{\boldsymbol{x}} \atop \tau \sim	\pi_{\theta}|\boldsymbol{y}_{0}=\boldsymbol{x}}  \left[ \sum_{t=1}^T \left( \mathop\mathbb{E}\limits_{y\sim \pi_*(\cdot|\boldsymbol{y}_{<t})}\!\! \left[f(\boldsymbol{y}_{<t},y) \right] - \mathop\mathbb{E}\limits_{y'\sim \pi_{\theta}(\cdot|\boldsymbol{y}_{<t}))} \left[f(\boldsymbol{y}_{<t}, y')\right] \right) \right] \\
	\overset{(i)}{\leq}& \mathop\mathbb{E}\limits_{\boldsymbol{x} \sim p_{\boldsymbol{x}} \atop \tau \sim \pi_{\theta}|\boldsymbol{y}_{0}=\boldsymbol{x}}  \left[ \sum_{t=1}^T \sup_{f:\Vert f \Vert_{\infty} \leq 1} \left( \mathop\mathbb{E}\limits_{y\sim \pi_*(\cdot|\boldsymbol{y}_{<t})}\!\! \left[f(\boldsymbol{y}_{<t},y) \right] - \mathop\mathbb{E}\limits_{y'\sim \pi_{\theta}(\cdot|\boldsymbol{y}_{<t}))} \left[f(\boldsymbol{y}_{<t}, y')\right] \right) \right] \\
	\overset{(ii)}{=}& \mathop\mathbb{E}\limits_{\boldsymbol{x} \sim p_{\boldsymbol{x}} \atop \tau \sim \pi_{\theta}|\boldsymbol{y}_{0}=\boldsymbol{x}}  \left[ \sum_{t=1}^T \sum_{y\in \mathcal{V}} \left\vert \pi_*(y|\boldsymbol{y}_{<t}) - \pi_{\theta}(y|\boldsymbol{y}_{<t}) \right\vert \right] \overset{\text{Def. \ref{def:stepdist}}}{=} d^{\rm on}_{\rm TV}(\pi_{\theta}, \pi_{*}),
\end{align*}
where $(i)$ follows from Jensen's inequality and $(ii)$ follows from \cite{sriperumbudur2009integral}.

Similarly, we can bound the off-policy version of Eq. (\ref{eq:pro1}) as follows,
\begin{align*}
	\sup_{f:\Vert f \Vert_{\infty} \leq 1} \mathop\mathbb{E}\limits_{\boldsymbol{x} \sim p_{\boldsymbol{x}} \atop \tau \sim \pi_*|\boldsymbol{y}_{0}=\boldsymbol{x}} \left[ \mathcal{U}^{\rm off}(\tau,f,\theta) \right] &\leq  \mathop\mathbb{E}\limits_{\boldsymbol{x} \sim p_{\boldsymbol{x}} \atop \tau \sim \pi_*|\boldsymbol{y}_{0}=\boldsymbol{x}}  \left[ \sum_{t=1}^T \sum_{y\in \mathcal{V}} \left\vert \pi_*(y|\boldsymbol{y}_{<t}) - \pi_{\theta}(y|\boldsymbol{y}_{<t}) \right\vert \right] \\
	&\overset{\text{Def. \ref{def:stepdist}}}{=} d^{\rm off}_{\rm TV}(\pi_{\theta}, \pi_{*})
\end{align*}
\end{proof}

\subsection{Derivation of Policy Gradient} \label{apdx:policygradient}
\begin{proof}
Based on the definition of training objective in Eq. (\ref{eq:obj}), we have
\begin{align} \label{eq:policygrad}
\begin{split}
\nabla_{\theta} \mathcal{L}(\theta,\phi_1,\phi_2) &= \nabla_{\theta} \mathop\mathbb{E}\limits_{\boldsymbol{x} \sim p_{\boldsymbol{x}}} \left[ \mathop\mathbb{E}\limits_{\tau \sim \pi_*|\boldsymbol{y}_{0}=\boldsymbol{x}} \left[ \mathcal{U}^{\rm off}(\tau, f_{\phi_1}; \theta) \right] + \mathop\mathbb{E}\limits_{\tau' \sim	\pi_{\theta}|\boldsymbol{y}_{0}=\boldsymbol{x}} \left[ \mathcal{U}^{\rm on}(\tau', f_{\phi_2}) \right] \right] \\
&= \mathop\mathbb{E}\limits_{\boldsymbol{x} \sim p_{\boldsymbol{x}}} \left[ \mathop\mathbb{E}\limits_{\tau \sim \pi_*|\boldsymbol{y}_{0}=\boldsymbol{x}} \left[ \nabla_{\theta} \mathcal{U}^{\rm off}(\tau, f_{\phi_1}; \theta) \right] + \nabla_{\theta}\mathop\mathbb{E}\limits_{\tau' \sim \pi_{\theta}|\boldsymbol{y}_{0}=\boldsymbol{x}} \left[  \mathcal{U}^{\rm on}(\tau', f_{\phi_2}) \right] \right] \\
\end{split}
\end{align}

Based on the definition of $\mathcal{U}^{\rm off}(\tau, f_{\phi_1}, \theta)$ in Eq. (\ref{eq:pro1}), we can compute the gradient $\nabla_{\theta} \mathcal{U}^{\rm off}(\tau, f_{\phi_1}, \theta)$ for any $\tau \in {\mathcal{Y}}$ and $\phi_1 \in \Phi$ as follows,
\begin{align}\label{eq:grad1}
	\begin{split}
	&\nabla_{\theta} \mathcal{U}^{\rm off}(\tau, f_{\phi_1}, \theta) \\
	=& - \sum_{t=1}^T \nabla_{\theta} \mathop\mathbb{E}\limits_{y \sim \pi_{\theta}(\cdot|\boldsymbol{y}_{<t})} \left[f_{\phi_1}(\boldsymbol{y}_{<t}, y) \right] \\
	=& - \sum_{t=1}^T  \sum_{y \in \mathcal{V}} \pi_{\theta}(y|\boldsymbol{y}_{<t}) \nabla_{\theta}\log\pi_{\theta}(y|\boldsymbol{y}_{<t}) f_{\phi_1}(\boldsymbol{y}_{<t}, y) \\
	=& - \sum_{t=1}^T \mathop\mathbb{E}\limits_{y\sim \pi_{\theta}(\cdot|\boldsymbol{y}_{<t})} \nabla_{\theta}\log\pi_{\theta}(y|\boldsymbol{y}_{<t}) f_{\phi_1}(\boldsymbol{y}_{<t}, y) 
	\end{split}
\end{align}

then, based on the definition of $ \mathcal{U}^{\rm on}(\tau', f_{\phi_2})$ in Eq. (\ref{eq:pro2}), we have
\begin{align} \label{eq:grad2}
	\begin{split}
	&\nabla_{\theta}\mathop\mathbb{E}\limits_{\tau' \sim \pi_{\theta}|\boldsymbol{y}_{0}=\boldsymbol{x}} \left[  \mathcal{U}^{\rm on}(\tau', f_{\phi_2}) \right] \\
	=& \int \nabla_{\theta} \left[ p_{\theta}(\tau| \boldsymbol{y}_{0}=\boldsymbol{x}) \mathcal{U}^{\rm on}(\tau', f_{\phi_2}) \right] d\tau' \\
	=& \int \nabla_{\theta} p_{\theta}(\tau'| \boldsymbol{y}_{0}=\boldsymbol{x}) \mathcal{U}^{\rm on}(\tau', f_{\phi_2}) d\tau' \\
	=& \int p_{\theta}(\tau'| \boldsymbol{y}_{0}=\boldsymbol{x}) \nabla_{\theta} \log p_{\theta}(\tau'| \boldsymbol{y}_{0}=\boldsymbol{x}) \mathcal{U}^{\rm on}(\tau', f_{\phi_2}) d\tau' \\
	=& \mathop\mathbb{E}\limits_{\tau' \sim \pi_{\theta}|\boldsymbol{y}_{0}=\boldsymbol{x}} \left[ \nabla_{\theta} \log p_{\theta}(\tau'| \boldsymbol{y}_{0}=\boldsymbol{x})\mathcal{U}^{\rm on}(\tau', f_{\phi_2}) \right]  \\
	\overset{(i)}{=}&\mathop\mathbb{E}\limits_{\tau' \sim \pi_{\theta}|\boldsymbol{y}_{0}=\boldsymbol{x}} \left[ \sum_{t=1}^T \nabla_{\theta} \log \pi_{\theta}(y'_t|\boldsymbol{y}'_{<t}) \mathcal{U}^{\rm on}(\tau', f_{\phi_2}) \right],
    \end{split}
\end{align}
where $(i)$ uses the fact that the state-transition is deterministic, and thus $\nabla_{\theta} \log p_{\theta}(\tau'| \boldsymbol{y}'_{0}=\boldsymbol{x}) = \sum_{t=1}^T \nabla_{\theta} \log \pi_{\theta}(y'_t|\boldsymbol{y}'_{<t})$ for any $\tau' = \{\boldsymbol{y}'_{<t}, y'_t\}_{t=1,\ldots,T}$.

Combining Eq. (\ref{eq:policygrad}) with Eq. (\ref{eq:grad1}) and Eq. (\ref{eq:grad2}), we have
\begin{align*}
	\nabla_{\theta} \mathcal{L}(\theta,\phi_1,\phi_2) =  \mathop\mathbb{E}\limits_{\boldsymbol{x} \sim p_{\boldsymbol{x}}} \Bigg[&- \mathop\mathbb{E}\limits_{\tau \sim \pi_*|\boldsymbol{y}_{0}=\boldsymbol{x}} \left[ \sum_{t=1}^T  \mathop\mathbb{E}\limits_{y\sim \pi_{\theta}(\cdot|\boldsymbol{y}_{<t})}  \nabla_{\theta}\log\pi_{\theta}(y|\boldsymbol{y}_{<t}) f_{\phi_1}(\boldsymbol{y}_{<t}, y)  \right]  \\
	&+ \mathop\mathbb{E}\limits_{\tau' \sim \pi_{\theta}|\boldsymbol{y}'_{0} = \boldsymbol{x}} \left[ \sum_{t=1}^T \nabla_{\theta}\log\pi_{\theta}(y'_t|\boldsymbol{y}'_{<t}) \mathcal{U}^{\rm on}(\tau', f_{\phi_2}) \right] \Bigg]
\end{align*}

\end{proof}

\section{Experimental Setup} \label{apdx:hyperparams}
We use NVIDIA A40 GPUs with 40GB RAM to conduct all the experiments.
\subsection{Instruction-Following Experiments}
\noindent\textbf{Base models.}
We conduct experiments on both GPT-2 \cite{radford2019language} and OpenLLaMA \cite{geng2023openllama}. For the GPT-2 experiments, we use GPT-2 XL\footnote{\url{https://huggingface.co/openai-community/gpt2-xl}} with 1.5B parameters to construct the teacher policy and GPT-2\footnote{\url{https://huggingface.co/openai-community/gpt2}} with 117M parameters to construct the student policy. For the $Q$-value function, we use a GPT-2 model for feature extraction and add a linear layer with the size of 768 x 1 to output the logits of step $y_t$ as $f(\boldsymbol{y}_{<t}, y_t)$.
For the OpenLLaMA experiments, we use OpenLLaMA-7B\footnote{\url{https://huggingface.co/openlm-research/open_llama_7b}} with 6.7B parameters to construct the teacher policy and OpenLLaMA-3B\footnote{\url{https://huggingface.co/openlm-research/open_llama_7b}} with 2.7B parameters to construct the student model. We construct the $Q$-value function with the OpenLLaMA-3B model and a linear layer with the size of 3,200 x 1. 

\noindent\textbf{Training details.}
We fine-tune the OpenLLaMA-7B teacher model and the OpenLLaMA-3B student models on the corresponding supervised dataset with 10,000 steps. The GPT-2 teacher and student models use the fine-tuned checkpoints by Gu et al. \cite{gu2023minillm}. For the implementation of compared baselines, we use the code by Ko et al. \cite{ko2024distillm} and re-run the results.
The optimization protocol for KD training largely follows the previous work \cite{gu2023minillm,ko2024distillm}. In particular, we search for the learning rates among a finite set for each experiment to obtain the best result. The batch size for each experiments is seleted to make full use of the 40GB RAM of an A40 GPU. To handle the adversarial training, we choose the number of adversarial steps $K=5$ and the adversarial control factor $\alpha=0.1$ based on the development experiments. The hyperparameters for training are listed in Table \ref{tab:hyperinst}.
\begin{table*}[h]
	\centering
	\scalebox{.9}{
		\begin{tabular}{lcc}
			\toprule
			{\bf Hyperparameter} & {\bf GPT-2} & {\bf OpenLLaMA} \\
			\midrule
			max. training steps & 10,000 & 10,000 \\
			adv. step $K$ & 5 & 5 \\
			batch size (per GPU) & \{8, 16, 32\} &\{4, 8\} \\
			dropout rate & 0.1 & 0.1 \\
			adv. control factor ($\alpha$) & 0.1 & 0.1 \\
			%policy optimizer & AdamW & AdamW \\
			%$Q$ optimizer & AdamW & AdamW \\
			learning rate (lr) & \{5e$^{-5}$, 1e$^{-4}$,5e$^{-4}$\} & \{5e$^{-6}$, 1e$^{-5}$, 5e$^{-5}$\} \\
			warmup steps & 1,000 & 500 \\
			weight decay & 1e$^{-2}$ & 1e$^{-2}$ \\
			max length & 512 & 512 \\
			sampling top-p & 1.0 & 1.0 \\
			sampling temperature &1.0 & 1.0 \\
			evaluation & Greedy Sampling & Greedy Sampling \\
			\#GPUs & 2 & 4 \\
			\bottomrule
		\end{tabular}}
	\caption{Training hyperparameters for instruction-following experiments.}\label{tab:hyperinst}
\end{table*}

\subsection{Task-Specific Experiments}
\noindent\textbf{Base models.}
For the text summarization and commonsense reasoning experiments, we use T5-XL\footnote{\url{https://huggingface.co/google/t5-v1_1-xl}} with 2.8B parameters to construct the teacher policy and construct the student policy with T5-Large\footnote{\url{https://huggingface.co/google/t5-v1_1-large}} (770M parameters), T5-Base\footnote{\url{https://huggingface.co/google/t5-v1_1-base}} (220M parameters) and T5-Small\footnote{\url{https://huggingface.co/google/t5-v1_1-small}} (60M parameters). For the machine translation experiments, we use mT5-XL \cite{xue2020mt5} to construct the teacher policy and use mT5-Large/-Base/-Small to construct the student policy. The corresponding $Q$-value functions are constructed using the student model and a linear layer of size 1024x1, 768x1 and 512x1 for the large, base and small models, respectively. 

\noindent\textbf{Training details}
We initialize the corresponding teacher and student models using 10,000-step-fine-tuning checkpoints on the SAMSum dataset, 80,000-step-fine-tuning checkpoints on the IWSLT'17 (en-de) dataset and 3,000-step-fine-tuning checkpoints on the StrategyQA dataset.
We largely follow Ko et al. \cite{ko2024distillm} to set the hyperparameters for training. In particular, we search for the learning rate from a preset range to obtain the best result for each baseline and our method. The batch size is selected to make full use of the RAM of GPUs. We use a relatively larger maximum number of training steps for IWSLT'17 (en-de) experiments to satisfy sufficient convergences for the machine translation task. We use beam search for the evaluation on the IWSLT'17 (en-de) dataset.
\begin{table*}[h]
	\centering
	\scalebox{.9}{
		\begin{tabular}{lccc}
			\toprule
			{\bf Hyperparameter} & {\bf SAMSum} & {\bf IWSLT'17 (en-de)} & {\bf StrategyQA} \\
			\midrule
			max. training steps & 10,000 & 80,000  & 3,000\\
			adv. step $K$ & 5 & 2 & 5 \\
			batch size (per GPU) & \{16, 32, 64\} & \{16, 32, 64\} & \{16, 32, 64\} \\
			dropout rate & 0.0 & 0.3 & 0.1\\
			adv. control factor ($\alpha$) & 0.1 & 0.1 & 0.1 \\
			%policy optimizer & AdamW & AdamW & AdamW \\
			%$Q$ optimizer & AdamW & AdamW & AdamW \\
			learning rate (lr) & \{5e$^{-5}$, 1e$^{-4}$,5e$^{-4}$\}& \{1e$^{-4}$,5e$^{-4}$,1e$^{-3}$\}& \{1e$^{-4}$,5e$^{-4}$,1e$^{-3}$\} \\
			warmup steps & 1,000 & 4,000 & 300 \\
			weight decay & 1e$^{-2}$& 1e$^{-4}$ & 1e$^{-2}$\\
			max. length & 1024 & 512 & 1,024 \\
			sampling top-p & 1.0 & 1.0 & 1.0 \\
			sampling temperature & 1.0 & 1.0 & 1.0\\
			evaluation & Greedy sampling & Beam search & Greedy sampling\\
			\#GPUs & 2 & 4 & 1\\
			\bottomrule
	\end{tabular}}
	\caption{Training hyperparameters on three task-specific datasets.} \label{tab:hypertask}
\end{table*}

\section{Additional Results}
\subsection{Results Based on GPT-2} \label{apdx:gpt2}
In addition to the experimental results based on OpenLLaMA for instruction-following tasks, we also conduct experiments based on GPT-2. Results are illustrated in Table \ref{tbl:instfollow1}. Compared with current state-of-the-art KD approaches, our method achieves the best results on five datasets with both GPT-4 feedback and ROUGE-L evaluations.
\begin{table*}[h] 
	\centering
	\scalebox{.8}{
		\begin{tabular}{lcccccccc}
			\toprule
			\multirow{2}{*}{\bf Method}& \multicolumn{2}{c}{{DollyEval}} & \multicolumn{2}{c}{{SelfInst}} & \multicolumn{2}{c}{{VicunaEval}} & S-NI & UnNI \\
			\cmidrule(lr){2-3}\cmidrule(lr){4-5}\cmidrule(lr){6-7}\cmidrule(lr){8-8}\cmidrule(lr){9-9}
			& GPT4 & R-L & GPT4 & R-L & GPT4 & R-L & R-L & R-L \\
			\midrule
			{\it GPT-2 XL (teacher)} & {\it 45.5$_{\pm 0.7}$} & {\it 28.2$_{\pm 0.8}$}  & {\it  34.7$_{\pm 1.6}$} & {\it 14.3$_{\pm 0.2}$} &{\it 32.7$_{\pm 1.6}$}  &{\it 16.2$_{\pm 0.3}$}  &{\it  27.6$_{\pm 0.3}$} &{\it 32.2$_{\pm 0.3}$}  \\
			\midrule
			SFT (student) & 29.8$_{\pm 1.2}$ & 23.4$_{\pm 0.2}$ & 20.2$_{\pm 0.7}$ & 10.3$_{\pm 0.5}$ & 17.8$_{\pm 0.9}$ & 14.6$_{\pm 0.4}$ & 16.1$_{\pm 0.3}$ & 18.2$_{\pm 0.6}$ \\
			KD \cite{hinton2015distilling} & \textcolor{gray}{29.5}$_{\pm 0.8}$ &23.8$_{\pm 0.3}$& \textcolor{gray}{18.0}$_{\pm 1.0}$& 12.3$_{\pm 0.2}$& \textcolor{gray}{17.2}$_{\pm 0.7}$ & 15.2$_{\pm 0.4}$& 20.8$_{\pm 0.5}$& 22.5$_{\pm 0.3}$ \\
			SeqKD \cite{kim2016sequence}& 29.8$_{\pm 0.5}$& 24.2$_{\pm 0.2}$& \textcolor{gray}{18.2}$_{\pm 0.8}$ & 11.6$_{\pm 0.4}$ &18.2$_{\pm 0.7}$ &15.5$_{\pm 0.3}$ & \textcolor{gray}{15.5}$_{\pm 0.6}$& 20.1$_{\pm 0.1}$ \\
			ImitKD \cite{lin2020autoregressive} &\textcolor{gray}{26.4}$_{\pm 0.6}$& \textcolor{gray}{22.7}$_{\pm 0.5}$& \textcolor{gray}{18.2}$_{\pm 0.5}$& 11.5$_{\pm 0.4}$ & 18.6$_{\pm 0.4}$& 14.5$_{\pm 0.3}$& 18.2$_{\pm 0.3}$& 21.8$_{\pm 0.4}$ \\
			MiniLLM \cite{gu2023minillm} & 30.2$_{\pm 1.2}$ & 24.3$_{\pm 0.3}$& 20.5$_{\pm 0.3}$& \underline{13.2}$_{\pm 0.3}$& 20.5$_{\pm 0.7}$& 18.5$_{\pm 0.3}$  & 22.7$_{\pm 0.3}$& 23.5$_{\pm 0.2}$ \\
			GKD \cite{agarwal2024policy} & \textcolor{gray}{29.2}$_{\pm 0.6}$& 23.6$_{\pm 0.2}$& 20.7$_{\pm 0.5}$& 12.7$_{\pm 0.2}$& 20.2$_{\pm 0.6}$& 17.7$_{\pm 0.2}$& 25.1$_{\pm 0.3}$& 25.9$_{\pm 0.1}$ \\
			DistiLLM \cite{ko2024distillm} & \underline{31.2}$_{\pm 0.4}$& \underline{25.2}$_{\pm 0.4}$& \underline{21.7}$_{\pm 0.5}$ & 12.5$_{\pm 0.3}$ & \underline{22.5}$_{\pm 1.2}$ & \underline{19.2}$_{\pm 0.5}$ & \underline{27.7}$_{\pm 0.2}$& \underline{27.6}$_{\pm 0.4}$ \\
			{\bf ours} & {\bf 31.7$_{\pm 0.5}$}& {\bf 26.1$_{\pm 0.3}$}& {\bf 22.7$_{\pm 0.5}$} & {\bf 14.2$_{\pm 0.3}$}& {\bf 23.6$_{\pm 0.8}$} & {\bf 20.5$_{\pm 0.2}$}  & {\bf 28.6$_{\pm 0.2}$} & {\bf 29.9$_{\pm 0.5}$}  \\
			\bottomrule
	\end{tabular}}
	\caption{Comparison with state-of-the-art KD methods on the instruction-following dataset using fine-tuned GPT-2 XL (1.5B) as the teacher and fine-tuned GPT-2 (0.1B) as the student. We format {\bf the best}, \underline{the second best} and \textcolor{gray}{worse than SFT} results.}
	\label{tbl:instfollow1}
\end{table*}

\subsection{Comparisons on Step-Wise Distribution Distance} \label{apdx:performdist}
Figure \ref{fig:distance_inst} and Figure \ref{fig:distance_task} illustrate performance comparison with well-studied step-wise distribution distance, including KL, RKL, JS divergences and TV distances. Results show that the optimization of proposed moment-matching objectives outperforms other step-wise distribution distances via either on-policy distillation and off-policy distillation. Besides, jointly using on-policy and off-policy moment-matching further improves the performances and achieves the best results on five instruction-following datasets with KD from the OpenLLaMA-7B to OpenLLaMA-3B model, and achieves the best results on three task-specific datasets with KD from the (m)T5-XL to (m)T5-Base model.
\begin{figure}[h]
	\centering
	\subfigure[DollyEval.]{
		\begin{minipage}[h]{0.3\textwidth}
			\centering
			\includegraphics[width=4.5cm]{distancebar_dolly}
	\end{minipage}}
	\subfigure[SelfInst.]{
		\begin{minipage}[h]{0.3\textwidth}
			\centering
			\includegraphics[width=4.5cm]{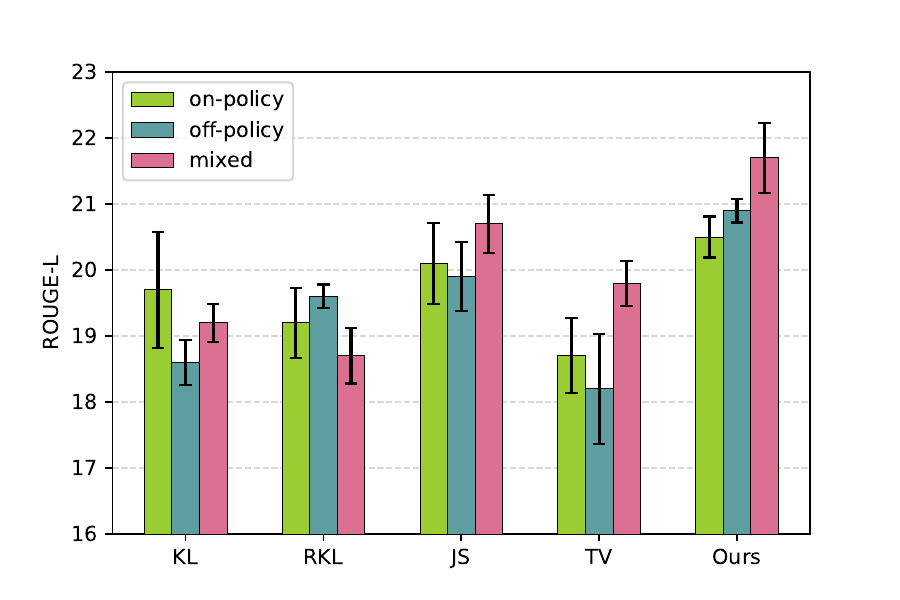}
	\end{minipage}}
	\subfigure[VicunaEval.]{
		\begin{minipage}[h]{0.3\textwidth}
			\centering
			\includegraphics[width=4.5cm]{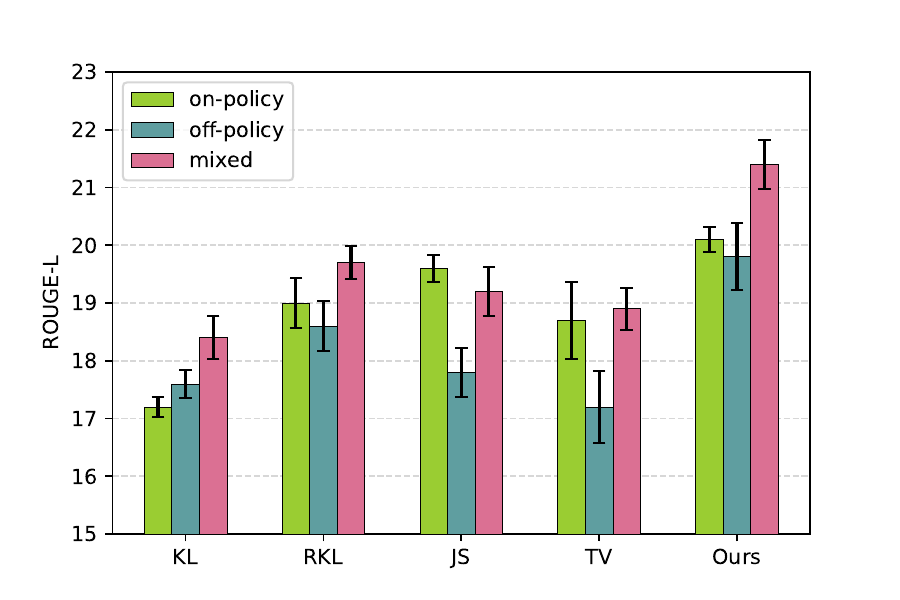}
	\end{minipage}}
	\subfigure[S-NI.]{
		\begin{minipage}[h]{0.3\textwidth}
			\centering
			\includegraphics[width=4.5cm]{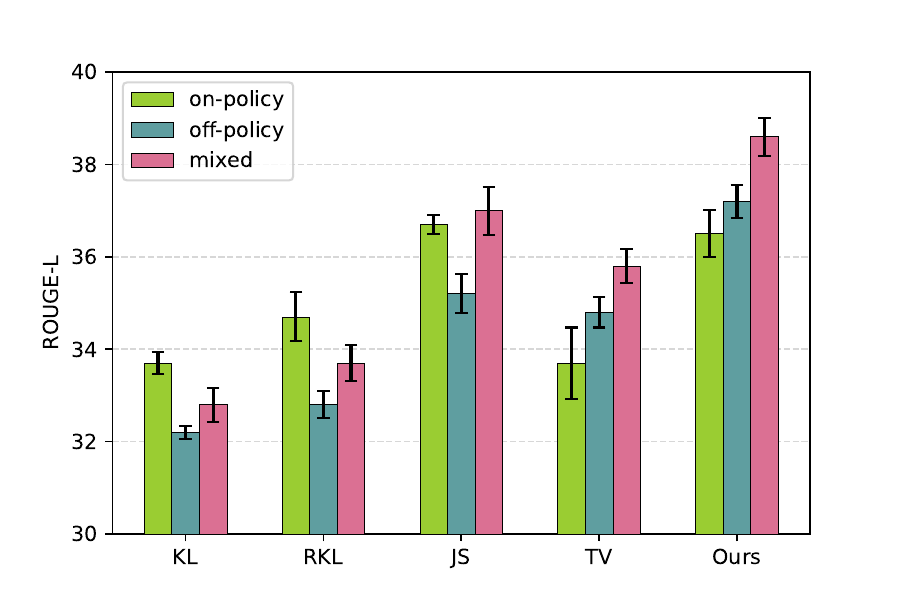}
	\end{minipage}}    
	\subfigure[UnNI.]{
	\begin{minipage}[h]{0.3\textwidth}
		\centering
		\includegraphics[width=4.5cm]{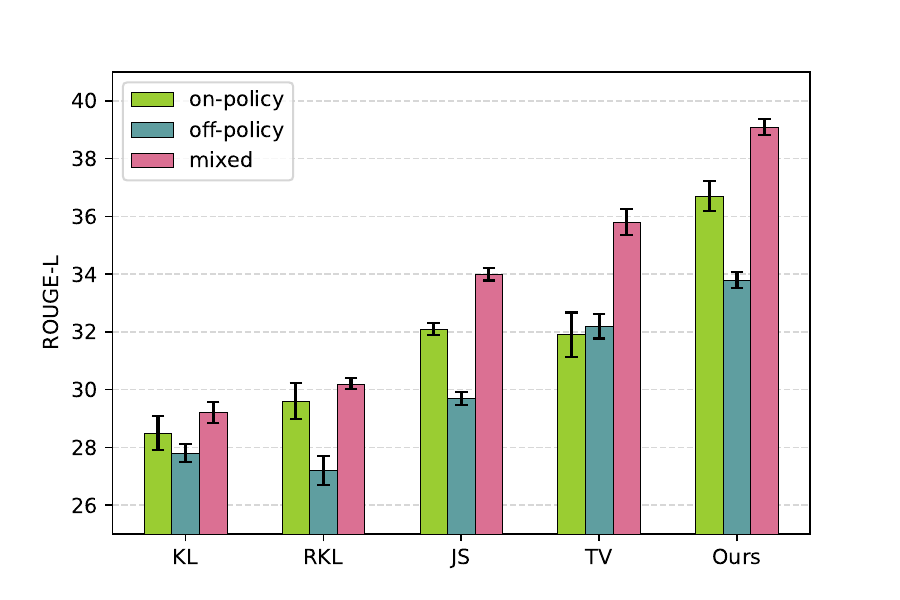}
    \end{minipage}}    
	\caption{Performance of difference step-wise distribution distances on five instruction-following datasets based on OpenLLaMA-3B $\rightarrow$ OpenLLaMA-7B.} \label{fig:distance_inst}
\end{figure}

\begin{figure}[h]
	\centering
	\subfigure[SAMSum.]{
		\begin{minipage}[h]{0.3\textwidth}
			\centering
			\includegraphics[width=4.5cm]{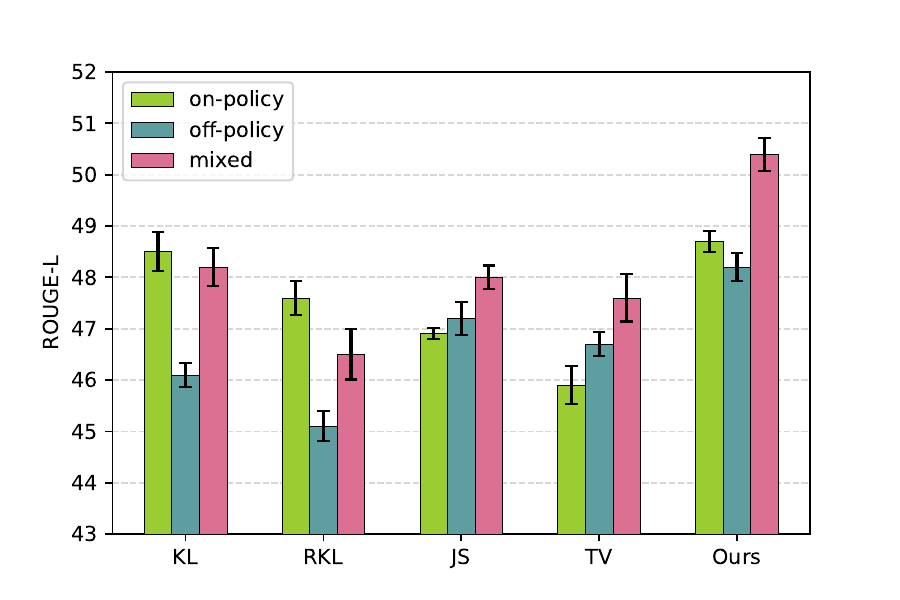}
	\end{minipage}}
	\subfigure[IWSLT'17 (en-de).]{
		\begin{minipage}[h]{0.3\textwidth}
			\centering
			\includegraphics[width=4.5cm]{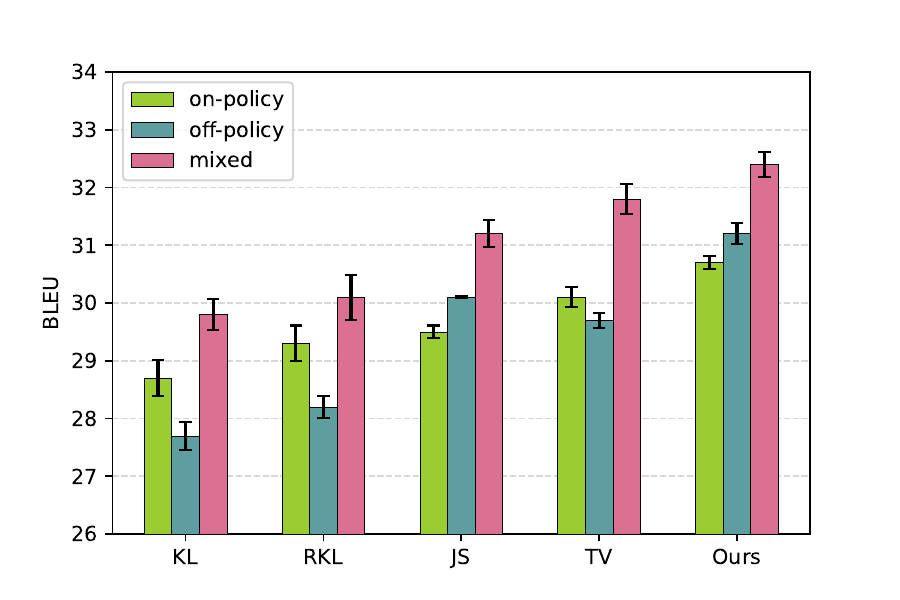}
	\end{minipage}}
	\subfigure[StrategyQA.]{
		\begin{minipage}[h]{0.3\textwidth}
			\centering
			\includegraphics[width=4.5cm]{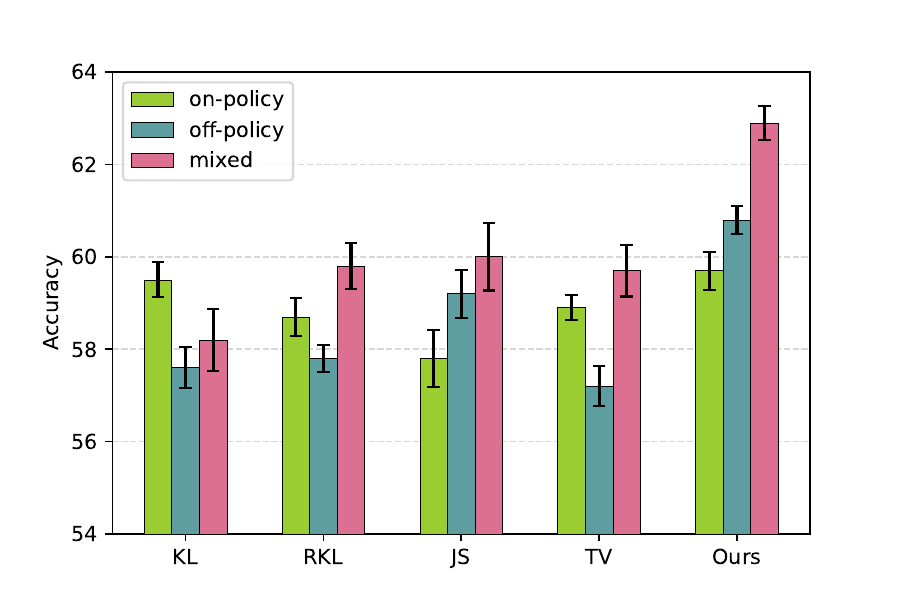}
	\end{minipage}}
	\caption{Performance of difference step-wise distribution distances on three task-specific datasets based on (m)T5-XL $\rightarrow$ (m)T5-Base.} \label{fig:distance_task}
\end{figure}

\subsection{Adversarial Training Procedure} \label{apdx:loss}
Figure \ref{fig:loss} illustrates the training loss and on-/off-policy moment-matching distances against the training steps on the instruction-following dataset and three task-specific datasets. We can observe that the training losses on four datasets have a similar trend, increasing at the beginning and then converging to a relatively lower level. The trend of loss function aligns with the characteristics of adversarial training with gradient descent ascent. In contrast, both the on-policy moment-matching distance $d^{\rm on}_{\rm MM}$ and the off-policy moment-matching distance $d^{\rm off}_{\rm MM}$ reduce as the number of training steps increases, which shows the effectiveness of our adversarial training approach for moment-matching.

\begin{figure}[h]
	\centering
	\subfigure[Instruct-following.]{
		\begin{minipage}[h]{0.45\textwidth}
			\centering
			\includegraphics[width=6.5cm]{loss_10000}
	\end{minipage}}
	\subfigure[SAMSum.]{
		\begin{minipage}[h]{0.45\textwidth}
			\centering
			\includegraphics[width=6.5cm]{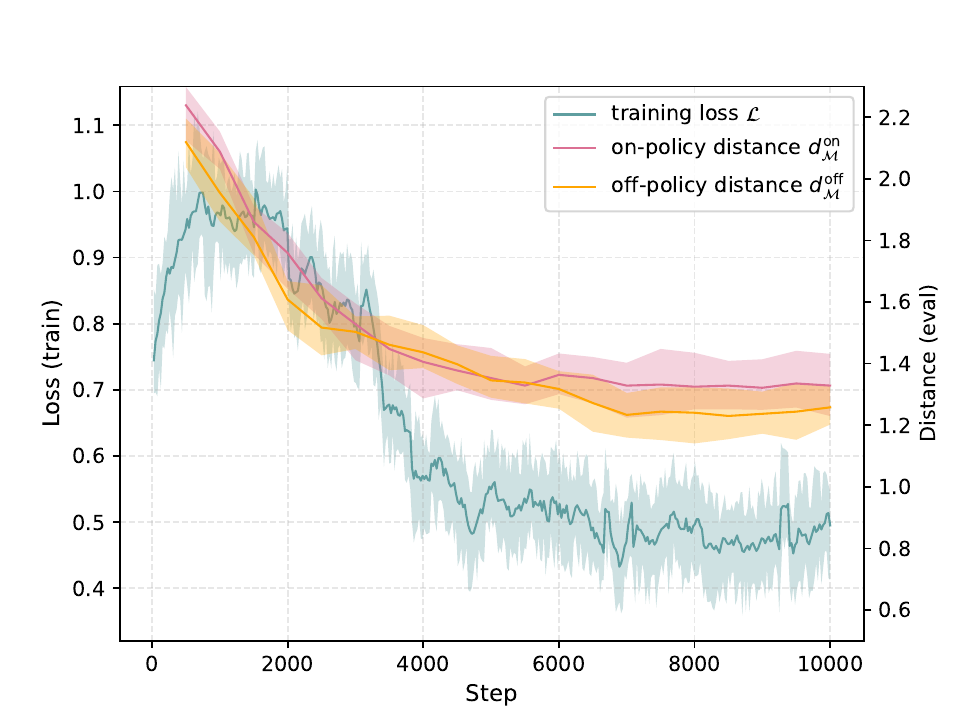}
	\end{minipage}}
	\subfigure[IWSLT'17 (en-de).]{
		\begin{minipage}[h]{0.45\textwidth}
			\centering
			\includegraphics[width=6.5cm]{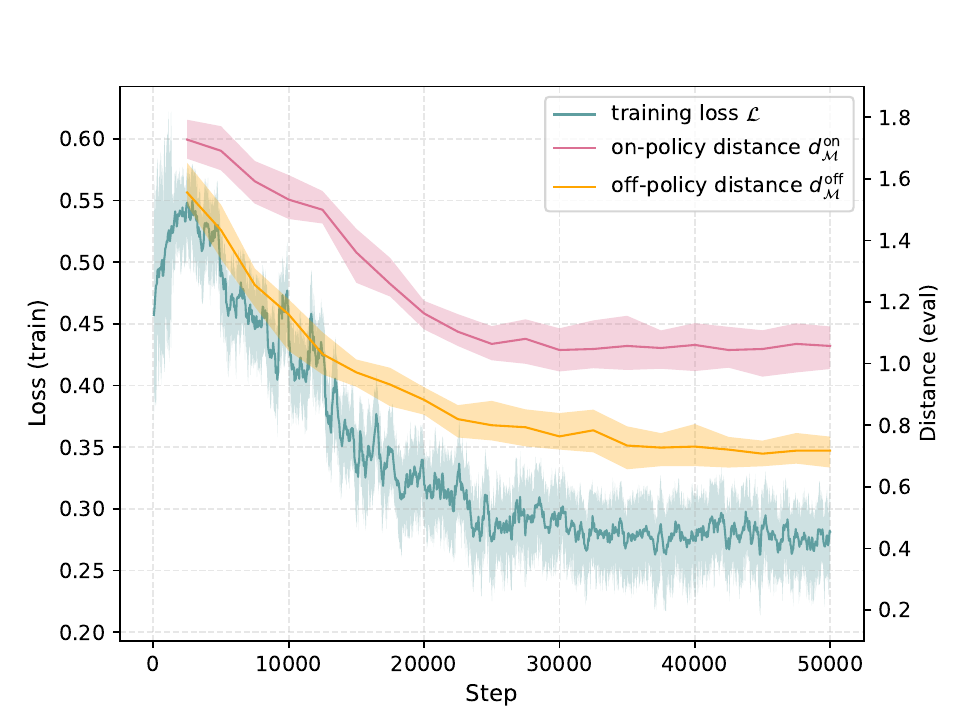}
	\end{minipage}}
	\subfigure[StrategyQA.]{
		\begin{minipage}[h]{0.45\textwidth}
			\centering
			\includegraphics[width=6.5cm]{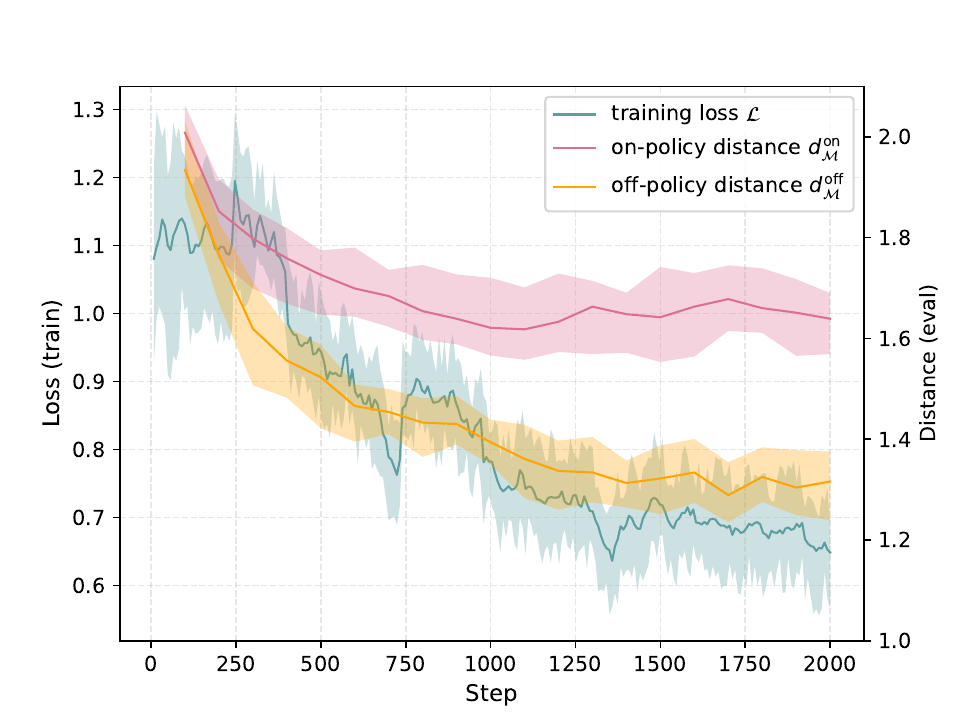}
	\end{minipage}}    
	\caption{Training loss and $d^{\rm on}_{\rm MM}$, $d^{\rm off}_{\rm MM}$ against training step on four datasets.} \label{fig:loss}
\end{figure}

\subsection{Moment-Matching via Distribution Matching} \label{apdx:dist}
We investigate how the distribution-matching methods via KL, RKL, JS divergences or TV distance can optimize the moment-matching distance in Figure \ref{fig:dist-inst} and Figure \ref{fig:dist-task}. Results show that the proposed adversarial training algorithm are more effective in minimizing the moment-matching distance than the distribution-matching methods.
\begin{figure}[h]
	\centering
	\subfigure[On-policy moment-matching distance.]{
		\begin{minipage}[h]{0.45\textwidth}
			\centering
			\includegraphics[width=6.5cm]{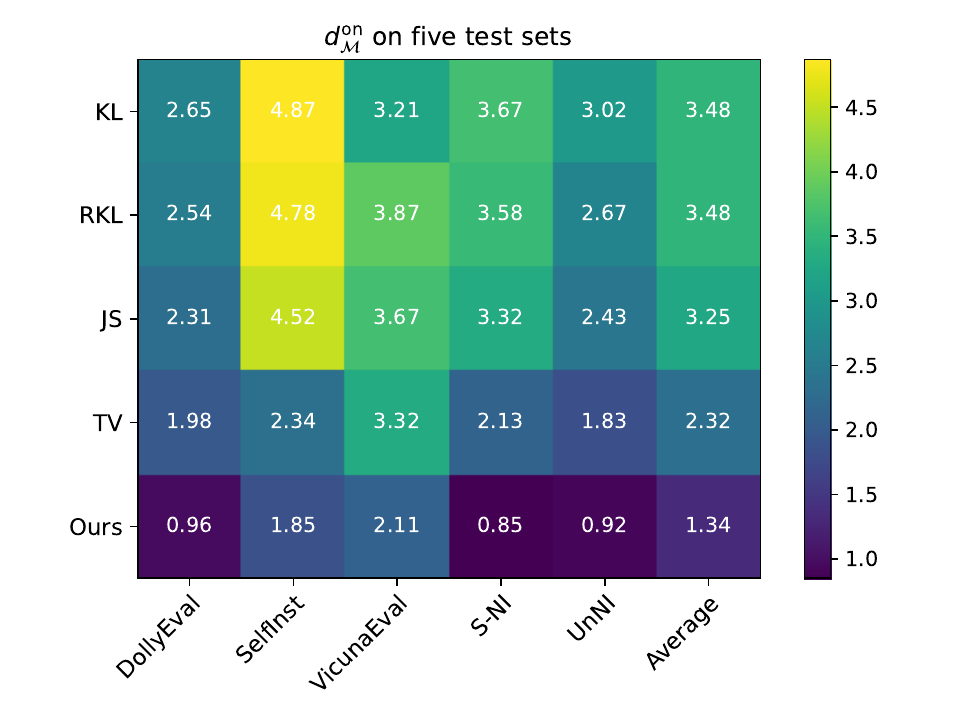}
	\end{minipage}}
	\subfigure[Off-policy moment-matching distance.]{
		\begin{minipage}[h]{0.45\textwidth}
			\centering
			\includegraphics[width=6.5cm]{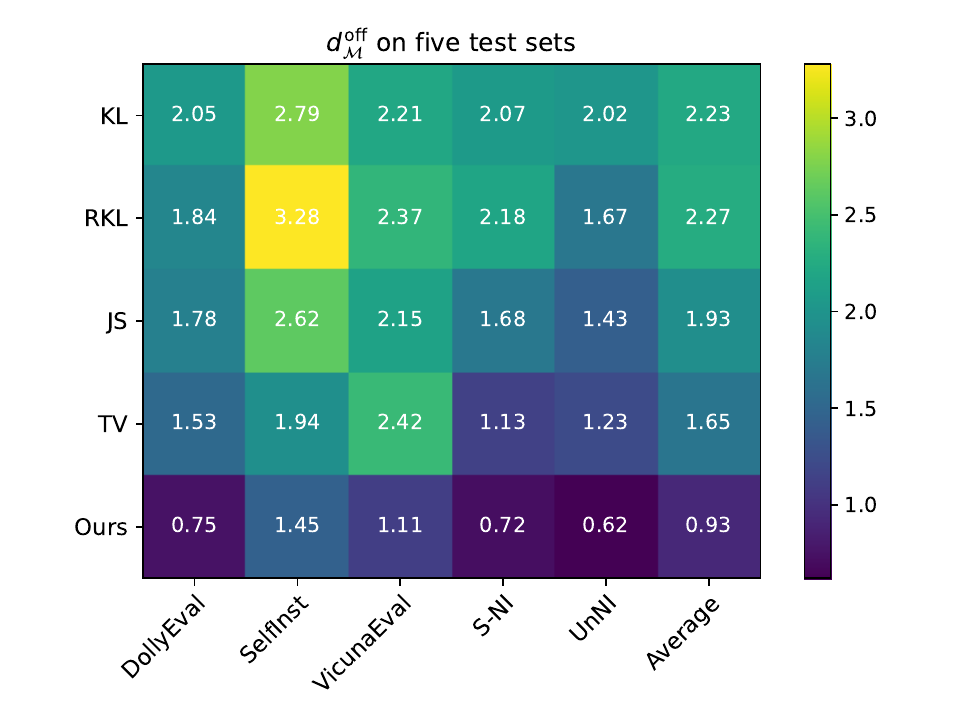}
	\end{minipage}} 
	\caption{Moment-matching via distribution-matching on the instruction-following dataset.} \label{fig:dist-inst}
\end{figure}

\begin{figure}[h]
	\centering
	\subfigure[On-policy moment-matching distance on SAMSum.]{
		\begin{minipage}[h]{0.3\textwidth}
			\centering
			\includegraphics[width=4.2cm]{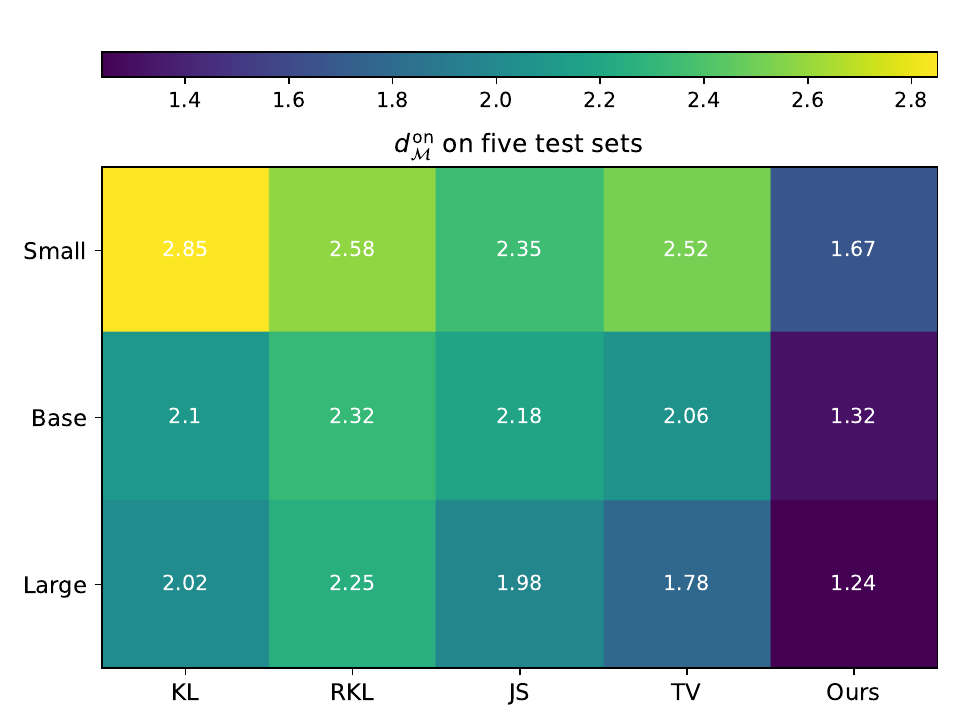}
	\end{minipage}}
	\subfigure[On-policy moment-matching distance on IWSLT'17 (en-de).]{
	\begin{minipage}[h]{0.3\textwidth}
		\centering
		\includegraphics[width=4.2cm]{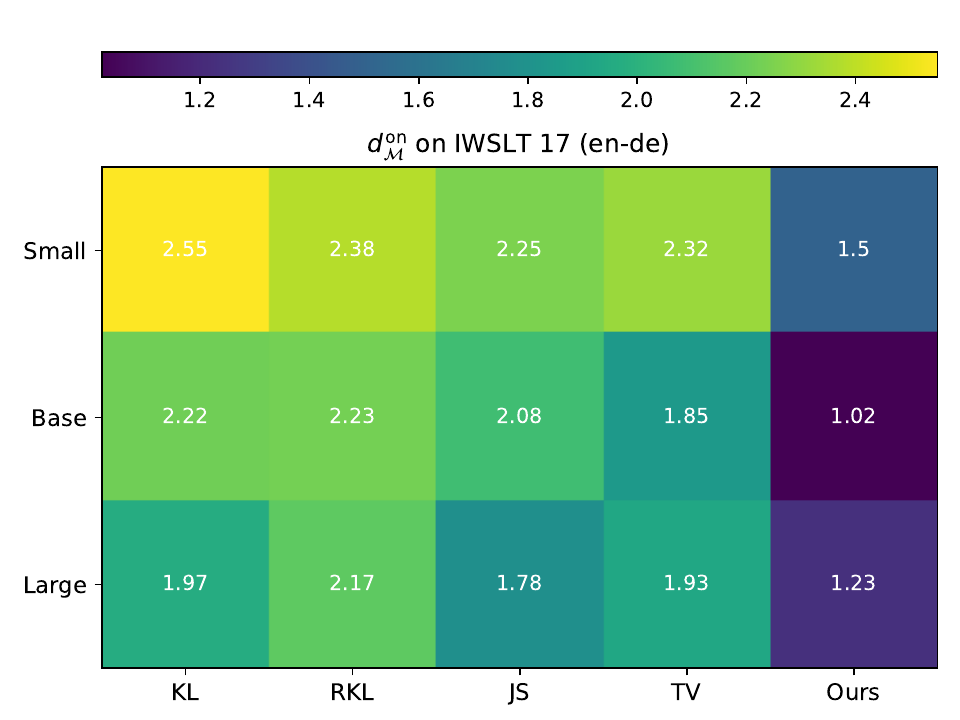}
	\end{minipage}} 
	\subfigure[On-policy moment-matching distance on StrategyQA.]{
	\begin{minipage}[h]{0.3\textwidth}
		\centering
		\includegraphics[width=4.2cm]{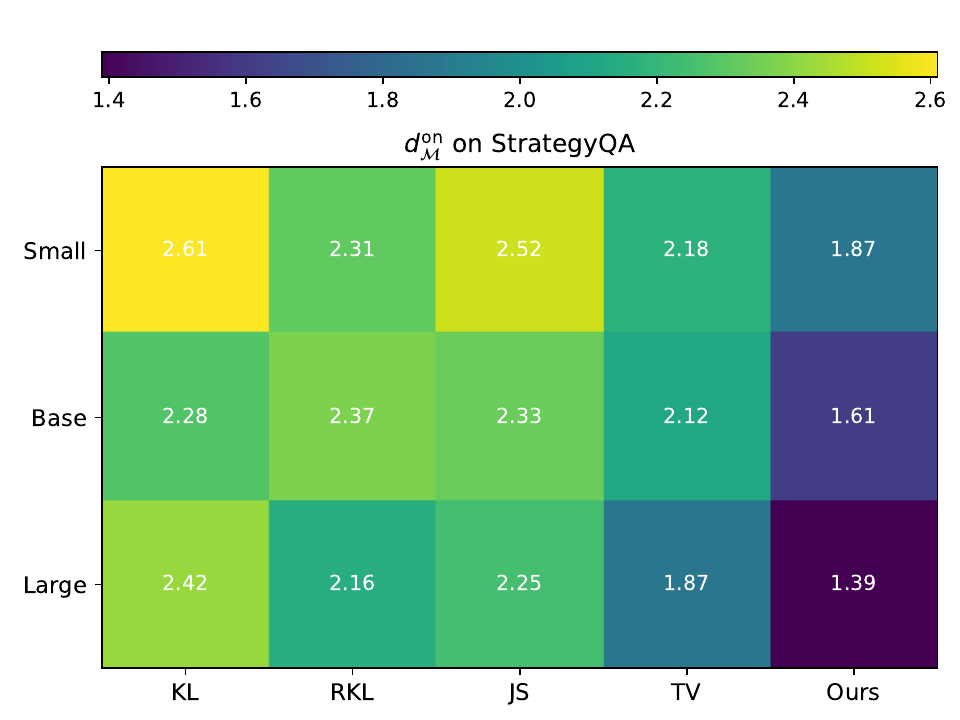}
	\end{minipage}}
	\subfigure[Off-policy moment-matching distance on SAMSum.]{
		\begin{minipage}[h]{0.3\textwidth}
			\centering
			\includegraphics[width=4.2cm]{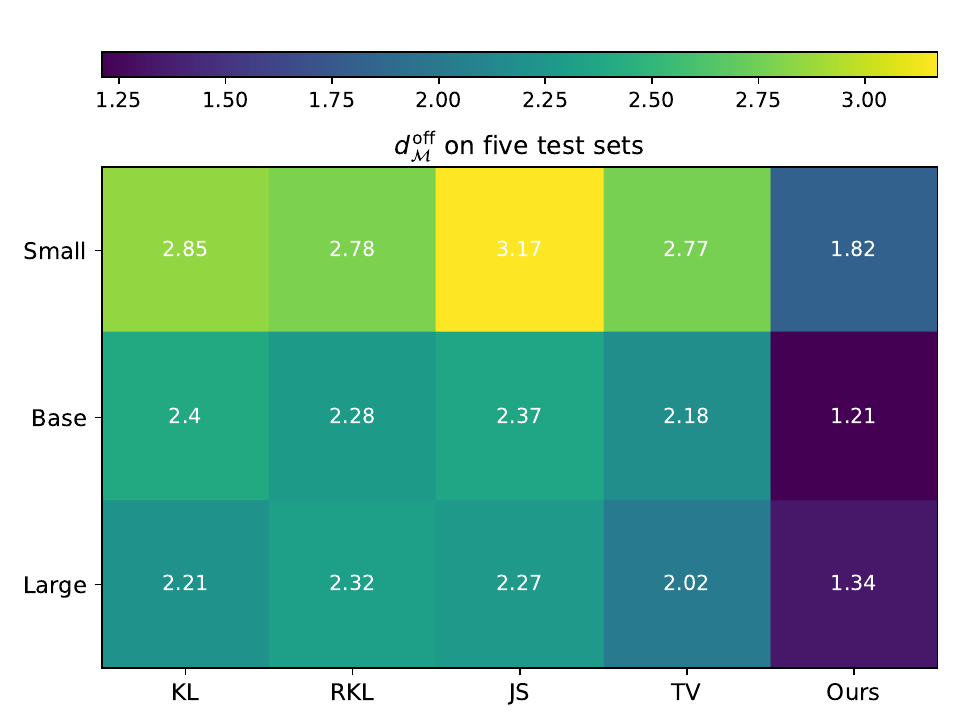}
	\end{minipage}} 
	\subfigure[Off-policy moment-matching distance on IWSLT'17 (en-de).]{
	\begin{minipage}[h]{0.3\textwidth}
		\centering
		\includegraphics[width=4.2cm]{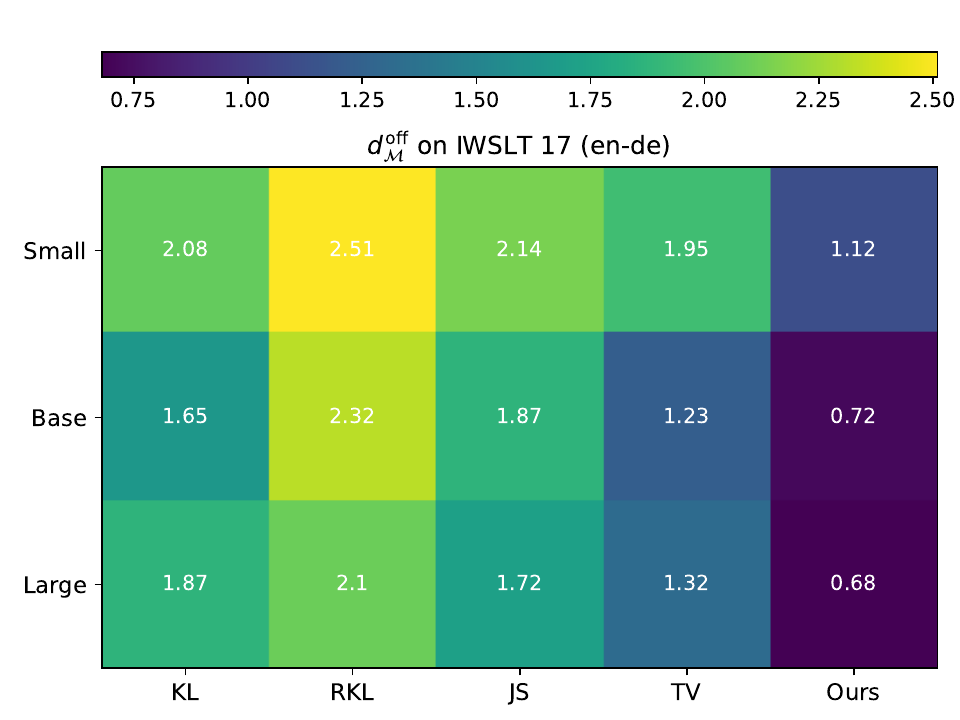}
	\end{minipage}}  
	\subfigure[Off-policy moment-matching distance on StrategyQA.]{
	\begin{minipage}[h]{0.3\textwidth}
		\centering
		\includegraphics[width=4.2cm]{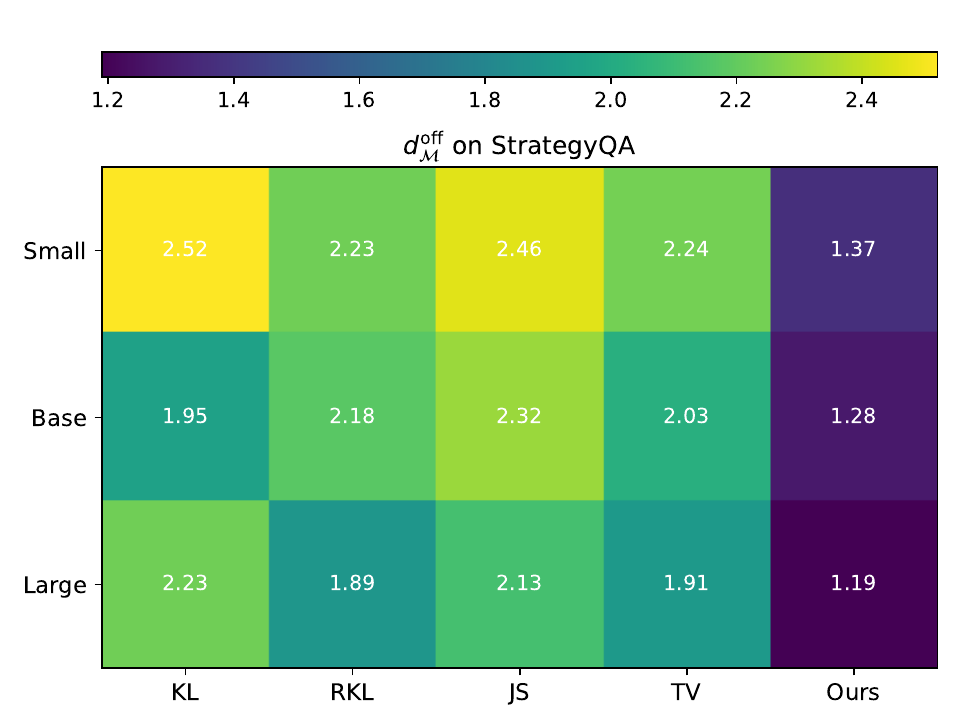}
	\end{minipage}} 
	\caption{Moment-matching via distribution-matching on three task-specific datasets.} \label{fig:dist-task}
\end{figure}

%%%%%%%%%%%%%%%%%%%%%%%%%%%%%%%%%%%%%%%%%%%%%%%%%%%%%%%%%%%%

%%%%%%%%%%%%%%%%%%%%%%%%%%%%%%%%%%%%%%%%%%%%%%%%%%%%%%%%%%%%

\end{document}